\documentclass[journal]{IEEEtran}
\usepackage{algorithmic}
\usepackage{algorithm}
\usepackage{array}
\usepackage[caption=false,font=normalsize,labelfont=sf,textfont=sf]{subfig}
\usepackage{textcomp}
\usepackage{stfloats}
\usepackage{url}
\usepackage{cite}
\usepackage{multirow}
\newcommand{\myparagraph}[1]{\vspace{0.1em}\noindent\textbf{#1}}
\usepackage[colorlinks,hyperindex,breaklinks]{hyperref}
\newcommand{\ie}{\textit{i}.\textit{e}.}
\newcommand{\eg}{\textit{e}.\textit{g}.}

\usepackage{xcolor}
\usepackage{colortbl}
\definecolor{ours}{gray}{.95}
\usepackage{pifont}
\newcommand{\cmark}{\ding{51}}
\newcommand{\xmark}{\ding{55}}
\usepackage{xcolor}
\usepackage{graphicx}
\usepackage{ragged2e}
\usepackage{arydshln}
\usepackage{amssymb}
\usepackage{amsmath}
\usepackage{amsthm}
\usepackage{float}
\usepackage{makecell}
\usepackage{booktabs}
\usepackage[english]{babel}

\newtheorem{theorem}{Theorem}

\newtheorem{proposition}{Proposition}
\definecolor{lightgreen}{RGB}{144,238,144} 
\definecolor{lightblue}{RGB}{58,121,150} 
\definecolor{bloodred}{RGB}{192,0,0} 
\begin{document}
\title{Towards Customized Knowledge Distillation for Chip-Level Dense Image Predictions}
\author{Dong~Zhang, Pingcheng Dong, Long Chen, Kwang-Ting~Cheng,~\IEEEmembership{Fellow,~IEEE}
\thanks{D. Zhang, P. Dong, and K-T. Cheng are with the Department of Electronic and Computer Engineering, HKUST, Hong Kong, China. E-mail:~\{dongz,~timcheng\}@ust.hk,~pingcheng.dong@connect.ust.hk.}
\thanks{L. Chen is with the Department of Computer Science and Engineering, HKUST, Hong Kong, China. E-mail:~longchen@ust.hk.}
}
\markboth{Under Submission}%
{Shell \MakeLowercase{\textit{et al.}}: Bare Demo of IEEEtran.cls for IEEE Journals}
\maketitle
\begin{abstract}
It has been revealed that efficient dense image prediction (EDIP) models designed for AI chips, trained using the knowledge distillation (KD) framework, encounter two key challenges, including \emph{maintaining boundary region completeness} and \emph{ensuring target region connectivity}, despite their favorable real-time capacity to recognize the main object regions. In this work, we propose a customized boundary and context knowledge distillation (BCKD) method for EDIPs, which facilitates the targeted KD from large accurate teacher models to compact small student models. Specifically, the \emph{boundary distillation} focuses on extracting explicit object-level boundaries from the hierarchical feature maps to enhance the student model's mask quality in boundary regions. Meanwhile, the \emph{context distillation} leverages self-relations as a bridge to transfer implicit pixel-level contexts from the teacher model to the student model, ensuring strong connectivity in target regions. Our proposed method is specifically designed for the EDIP tasks and is characterized by its simplicity and efficiency. Theoretical analysis and extensive experimental results across semantic segmentation, object detection, and instance segmentation on five representative datasets demonstrate the effectiveness of BCKD, resulting in well-defined object boundaries and smooth connecting regions.
\end{abstract}
\begin{IEEEkeywords}
Model compression, Knowledge distillation, Dense image predictions, Boundary and context learning.
\end{IEEEkeywords}

\IEEEpeerreviewmaketitle
\section{Introduction}
\label{intro}
\IEEEPARstart{T}{he} dense image prediction (DIP) tasks, \eg, semantic segmentation~\cite{long2015fully}, object detection~\cite{girshick2015fast}, and instance segmentation~\cite{wang2021end}, are fundamental yet challenging research problems within both domains of computer vision and multimedia computing~\cite{zhang2020causal,ahn2019weakly}. 
The objective of these tasks is to assign a semantic label to each object and/or pixel of the given image~\cite{zhang2020causal}. In recent years, advancements in general-purpose GPU technology have resulted in notable enhancements in both size and accuracy of sophisticated DIP models~\cite{cao2022swin,strudel2021segmenter}, \eg, Mask2Former~\cite{cheng2022masked}, SegNeXt~\cite{guo2022segnext}, and SAM~\cite{kirillov2023segment}. However, deploying large and accurate DIP models on resource-constrained edge computing devices, \eg, artificial intelligence chips~\cite{dong202528nm}, presents significant challenges due to the substantial computational costs and high memory consumptions associated with these models~\cite{wang2021end}. 
\begin{figure*}[t]
\centering
\includegraphics[width=1\textwidth]{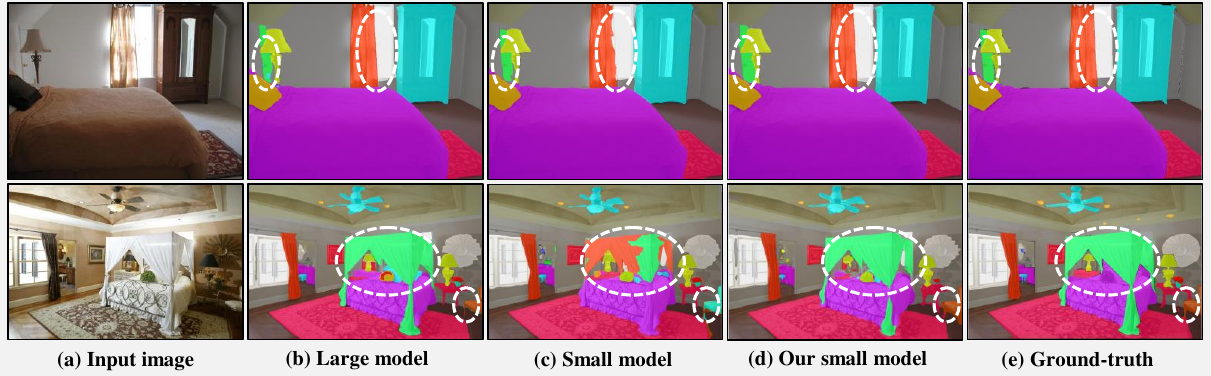}
\vspace{-6mm}
\caption{Two representative cases that small models are prone to produce errors. Result comparisons between large accurate models (b) and small efficient models (c) show that the latter tend to make errors in \textsl{maintaining boundary region completeness} (\eg, the \textit{\textcolor{orange}{``curtain''}} and the \textit{\textcolor{green}{``door''}}) and \textsl{preserving target region connectivity} (\eg, the \textit{\textcolor{lightgreen}{``bed valance''}} and the \textit{\textcolor{red}{``chair''}}). With the help of our BCKD in (d), small  models can address the two types of errors, leading to crisp region boundaries and smooth connecting regions. ``\textsl{w/}'' denotes with the corresponding implementation. Samples are from the ADE20K dataset~\cite{zhou2017scene}.}
\label{fig1}
\vspace{-3mm}
\end{figure*}

Compressing large DIP models into compact efficient DIP (EDIP) models offers an intuitive and cost-effective solution to address the severe resource limitations associated with mapping vision models onto edge computing devices~\cite{dong20214,zhang2021self}. In particular, the cross-architecture manner enables compressed models to seamlessly adapt to customized edge devices, eliminating the need for hardware modifications while maintaining computational efficiency. This manner significantly reduces deployment complexity and enhances the flexibility of model inference across heterogeneous edge computing platforms~\cite{yin20231}.
To achieve this goal and develop accuracy-preserving EDIP models, knowledge distillation (KD)~\cite{hinton2015distilling,wang2023efficient}, a prevalent model compression technology, has been pragmatically employed by using a small efficient model~(\ie, the student model) by imitating the behavior of a large accurate model~(\ie, the teacher model) in training~\cite{hinton2015distilling,zhao2022decoupled}. During inference, only the student model is utilized, allowing for a highly-efficient recognition pattern while simultaneously reducing the model size~\cite{dong20214,zhang2021self,cui2023kd}. 
Despite significant advancements by current KD methods across multiple dimensions, including sophisticated distillation strategies~\cite{cui2023kd} and complex distillation content~\cite{wang2020intra}, the inherent complexity of DIP continues to pose two critical challenges for existing approaches, particularly for the efficient compact models. The details are as follows: 

Primary KD methods mainly emphasize the imitation of general knowledge (\eg, features, regions, and logits) while overlooking the nuanced understanding of features along objective semantic boundaries and connecting internal regions essential for EDIPs~\cite{zhang2021self,wang2022active}. Particularly, since the small student model often predicts the main object regions fairly well but fails in boundary and connecting regions~\cite{fu2019dual,yuan2020object,cao2022swin}, the conventional utilization of task-agnostic general KD may not be effective enough and can be considered purposeless and redundant~\cite{gou2021knowledge,xu2020feature,zhao2022decoupled}, remaining a performance gap between the obtained results and the expected ones~\cite{wang2024crosskd}. 
For instance, we recommend the representative semantic segmentation task as an example. As shown in Figure~\ref{fig1}, the small student model (\ie, PSPNet-18~\cite{zhao2017pyramid}) in (c) produces inferior results compared to the large teacher model (\ie, PSPNet-101~\cite{zhao2017pyramid}) results in (b). The student model wrongly segments the boundary regions of \textit{\textcolor{orange}{``curtain''}} and \textit{\textcolor{green}{``door''}} as the background category or other foreground objects, and produces fragmented \textit{\textcolor{lightgreen}{``bed valance''}} and \textit{\textcolor{red}{``chair''}}, breaking the regional {relation connectivity}. Generally, the common errors observed in the outputs of the small student model can be summarized as \emph{maintaining boundary region completeness} and \emph{ensuring target region connectivity}.

To mitigate these errors and narrow the performance gap, in this paper, we propose a customized and targeted KD strategy termed as \textbf{B}oundary and \textbf{C}ontext \textbf{K}nowledge \textbf{D}istillation (BCKD). By ``customized'', we mean that our method's inherent ability to synergistically address the common errors present in existing EDIP models, while also coexisting with other methods (\emph{ref.}~Sec.~\ref{ablation}). BCKD mainly consists of two key components: the \emph{boundary distillation} and the \emph{context distillation}, aimed at rectifying the typical common errors encountered by EDIP models in \emph{maintaining boundary region completeness} and \emph{ensuring target region connectivity}, respectively. Specifically, \emph{boundary distillation} involves generating explicit object-level boundaries from the hierarchical backbone features, enhancing the completeness of the student model's masks in boundary regions (\emph{ref.}~Sec.~\ref{sec:4:2}). At the same time, \emph{context distillation} transfers implicit pixel-level contexts from the teacher model to the student model through self-relations, ensuring robust connectivity in the student's masks (\emph{ref.}~Sec.~\ref{sec:4:3}). BCKD is tailored specifically for EDIP tasks and offers a more targeted distillation pattern and a more tailored distillation manner compared to conventional task-agnostic KD methods. From a rigorous theoretical perspective, we establish and prove the effectiveness of our BCKD method (\emph{ref.}~Sec.~\ref{sec:4:5}). To validate the superior accuracy, we conducted extensive experiments in three representative dense image prediction tasks, including semantic segmentation, object detection, and instance segmentation, utilizing five challenging datasets such as Pascal VOC 2012~\cite{everingham2010pascal}, Cityscapes~\cite{cordts2016cityscapes}, ADE20K~\cite{zhou2017scene}, COCO-Stuff 10K~\cite{caesar2018coco}, and MS-COCO 2017~\cite{lin2014microsoft}. Qualitatively, BCKD produces sharp region boundaries and smooth connecting regions, addressing challenges that have hindered EDIP models. Quantitatively, BCKD consistently improves the accuracy of baseline models in various metrics, achieving competitive performance.

The main contributions are summarized as the following three folds: \emph{\textbf{1})}~We revealed two prevalent issues in existing EDIP models: maintaining boundary region completeness and ensuring target region connectivity. \emph{\textbf{2})}~We proposed a customized and targeted boundary and context knowledge distillation method. Our method not only demonstrates inherent coherence, but also possesses the capability to coexist with other methods. 
\emph{\textbf{3})}~Theoretical analysis demonstrates the superior effectiveness of our BCKD.
\emph{\textbf{4})}~Experimental evaluations across various tasks, baselines and datasets illustrate the superior accuracy of our method in comparison with existing methods. 

\section{Related Work}
\subsection{Dense Image Prediction (DIP) Tasks} 
DIP is a fundamental research problem within the fields of computer vision and multimedia computing, with the objective of assigning each object and/or pixel in an input image to a predefined category label, thereby enabling comprehensive semantic image recognition~\cite{long2015fully,zhou2024boundary,zhang2020causal,lin2023click}. 
Current mainstream DIP models can be roughly classified into the following three categories based on their backbone components: 1)~methods based on CNNs~\cite{long2015fully,yu2018bisenet,noh2015learning,huang2019ccnet}, 2)~methods based on ViT\footnote{We consider the visual state space model-based methods as a specialized Transformer architecture~\cite{gu2023mamba,xu2024survey}, owing to its structural similarities with the ViT model. Besides, we do not address the content related to these models. Therefore, we will no longer have separate discussions on this aspect.}~\cite{strudel2021segmenter,wang2022uformer,zheng2021rethinking}, and 3)~methods that combine CNNs and ViT~\cite{li2022next,peng2021conformer}. The key difference between these types of architectures is the approach used for feature extraction and how the extracted features are utilized in enhancing the capacity of CNNs models to capture contextual features~\cite{cao2022swin,zhang2025generalized}, increasing the capacity of ViT models to capture local features~\cite{wu2021cvt,zhang2022graph,peng2021conformer}, and leveraging low-level features to improve representation capacity~\cite{zhang2023cae,xie2021segformer} for achieving favorable results. 

Concretely, due to differences in feature extraction manners between CNNs and ViT, these two categories exhibit slight performance differences~\cite{wang2022uformer,peng2021conformer,zheng2021rethinking,li2022next}. For example, CNNs methods are better at predicting local object regions, while ViT methods, due to their stronger contextual information, can produce more complete object masks. Fortunately, the mixture of CNNs and ViT (\eg, CMT~\cite{guo2022cmt}, CvT~\cite{wu2021cvt}, ConFormer~\cite{peng2021conformer}, CAE-GreaT~\cite{zhang2023cae}, and visual Mamba~\cite{gu2023mamba}) uses the representation strengths of both patterns, resulting in highly satisfactory recognition performance~\cite{han2022survey,mao2022towards,wang2021pyramid}. 
In addition to these fundamental categories, there are advanced approaches that also utilize task-specific training tricks (\eg, graph reasoning~\cite{zhang2022graph}, linear attention~\cite{kong2022spvit}, mult-scale representation~\cite{fan2021multiscale}) to improve the accuracy. However, while current methods have achieved promising accuracy, mapping these models on resource constrained edge computing devices remains challenging because these devices typically have limited computation resources and memory consumptions~\cite{dong20231920,dong202528nm}. In this work, we do not intend to modify the network architecture. We first investigate the result disparities between small and large models and then propose a novel KD strategy tailored to the EDIP models. We aim at improving the recognition accuracy of small models without requiring any extra training data or increasing the inference costs.

\subsection{Knowledge Distillation (KD) in DIPs} 
KD is a well-established model compression technology that aims to transfer valuable knowledge from a large accuracy teacher model to a small efficient student model, with the objective of enhancing the student's accuracy during inference~\cite{chen2023kbstyle,gou2023multi,wang2020intra}. 
It is worth mentioning that the effectiveness of KD in cross-architecture scenarios has enabled significant flexibility in artificial intelligence chip design, as it eliminates the need to modify the underlying operators while maintaining model performance, which provides a practical solution for hardware adaptation without compromising computational efficiency~\cite{dong202528nm}. The key factors for the success of KD in DIPs are: {1)~the types of knowledge being distilled}, \eg, general knowledge: features and logits, and task-specific knowledge: class edge for semantic segmentation and object localization for object detection, {2)~the distillation strategies employed}, \eg, offline distillation~\cite{tseng2022offline}, online distillation~\cite{guo2020online}, and self-distillation~\cite{zhang2019your}, and {3)~the architecture of the teacher-student pair}, \eg, multi-teacher KD~\cite{yuan2021reinforced}, attention-based KD~\cite{passban2021alp}, and graph-based KD~\cite{lee2019graph}.

While the effectiveness of existing KD methods has been validated in general vision tasks, current methods mainly rely on coarse task-agnostic knowledge and do not consider the task-specific feature requirements~\cite{sun2020distilling,gou2021knowledge,mao2022towards,chen2017learning,wang2019distilling}. Especially for EDIPs, models are highly sensitive to feature representations~\cite{long2015fully,zhang2022unabridged,zheng2021rethinking,zhang2018context}. Therefore, general KD may not be effective enough and can be considered purposeless and redundant~\cite{gou2021knowledge,xu2020feature,zhao2022decoupled,zhang2022graph}, remaining a performance gap between the obtained results and the expected ones~\cite{yang2022cross,cui2023kd,wang2024crosskd}. Recent studies have shown that task-specific patterns of KD can help further improve the performance of student models~\cite{zheng2022LD}. For example, in object detection, object localization KD leads to more accurate predictions than the general knowledge~\cite{sun2020distilling,zhixing2021distilling}. In this work, we also adopt the idea of task-specific KD. Our contribution lies in proposing a customized KD scheme, namely {boundary distillation} and {context distillation}, which target the common errors of EDIP models, namely their tendency to make errors in \emph{maintaining boundary region completeness} and \emph{ensuring target region connectivity}. It is also worth noting that while some advanced methods, \eg, CTO~\cite{lin2023rethinking}, SlimSeg~\cite{xue2022slimseg}, and BPKD~\cite{liu2024bpkd}, have integrated boundary information in EDIPs, \emph{they necessitate the pre-extraction and incorporation of ground-truth boundary information}~\cite{xue2022slimseg,liu2024bpkd,lin2023rethinking}. In contrast, our method obviates the demand for pre-extracting object boundaries, thereby making it more practical for real-world applications and enabling savings in time and labor.
\section{Preliminaries}
In the training phase, KD intends to facilitate expectant knowledge transfer from a large teacher model $\mathbb{T}$ to a small compact student model $\mathbb{S}$, with the primary goal of enhancing the accuracy of $\mathbb{S}$ ~\cite{hinton2015distilling,zhang2021self,zhang2019your,gou2021knowledge,ji2021refine}. In inference, only $\mathbb{S}$ is used, so there are no computational overheads. 
Typically, features and logits serve as a medium for knowledge transfer. Besides, the temperature scaling strategy is often utilized to smooth the features and logits, which helps to lower prediction confidence and alleviate the issue of excessive self-assurance in $\mathbb{T}$~\cite{phuong2019distillation}. Formally, KD can be expressed by minimizing the cross-entropy loss as follows:
\begin{equation}
\mathcal{L}_{KD} = - \tau^2 \sum_{i \in M} {\sigma(\textbf{T}_i)^{1/\tau}\textrm{log}\left(\sigma(\textbf{S}_i)^{1/\tau} \right)}, 
\label{eq1}
\end{equation}
where $\textbf{T}_{i}$ and $\textbf{S}_{i}$ (both have been adjusted to the same dimension via $1 \times 1$ convolutions) 
are the $i$-th feature/logit item extracted from $\mathbb{T}$ and $\mathbb{S}$, respectively. $\sigma(\cdot)$ is the softmax normalization operation along the channel dimension, and $\tau \in \mathbb{R}_+$ denotes the temperature scaling coefficient. $M$ denotes the learning objective of $\textbf{T}_{i}$ and $\textbf{S}_{i}$, which typically refers to the spatial dimensions. Following~\cite{zhang2019your,phuong2019distillation}, to simplify temperature scaling effectively, we use $\textbf{T}_{i}$/$\textbf{S}_{i}$ divided by $\tau$ to achieve the similar effect. In addition to cross-entropy loss, other loss functions are also commonly used for KD, including KL divergence loss and MSE loss~\cite{liu2019structured,liu2022transkd,zheng2023distilling}. While existing KD methods have demonstrated promising results across various vision tasks~\cite{zhang2021self,cui2023kd,xu2020feature,liu2019structured,wang2020intra,gou2021knowledge}, they have not adequately addressed the specific feature understanding required for object boundaries and connecting regions in EDIP tasks. This oversight has led to suboptimal performance in these contexts. In following, we will introduce a complementary and targeted KD scheme informed by the common failure cases observed in small models, as shown in Figure~\ref{fig1}, with the aim of enhancing their inference accuracy.
\begin{figure*}[t]
\centering
\includegraphics[width=1\textwidth]{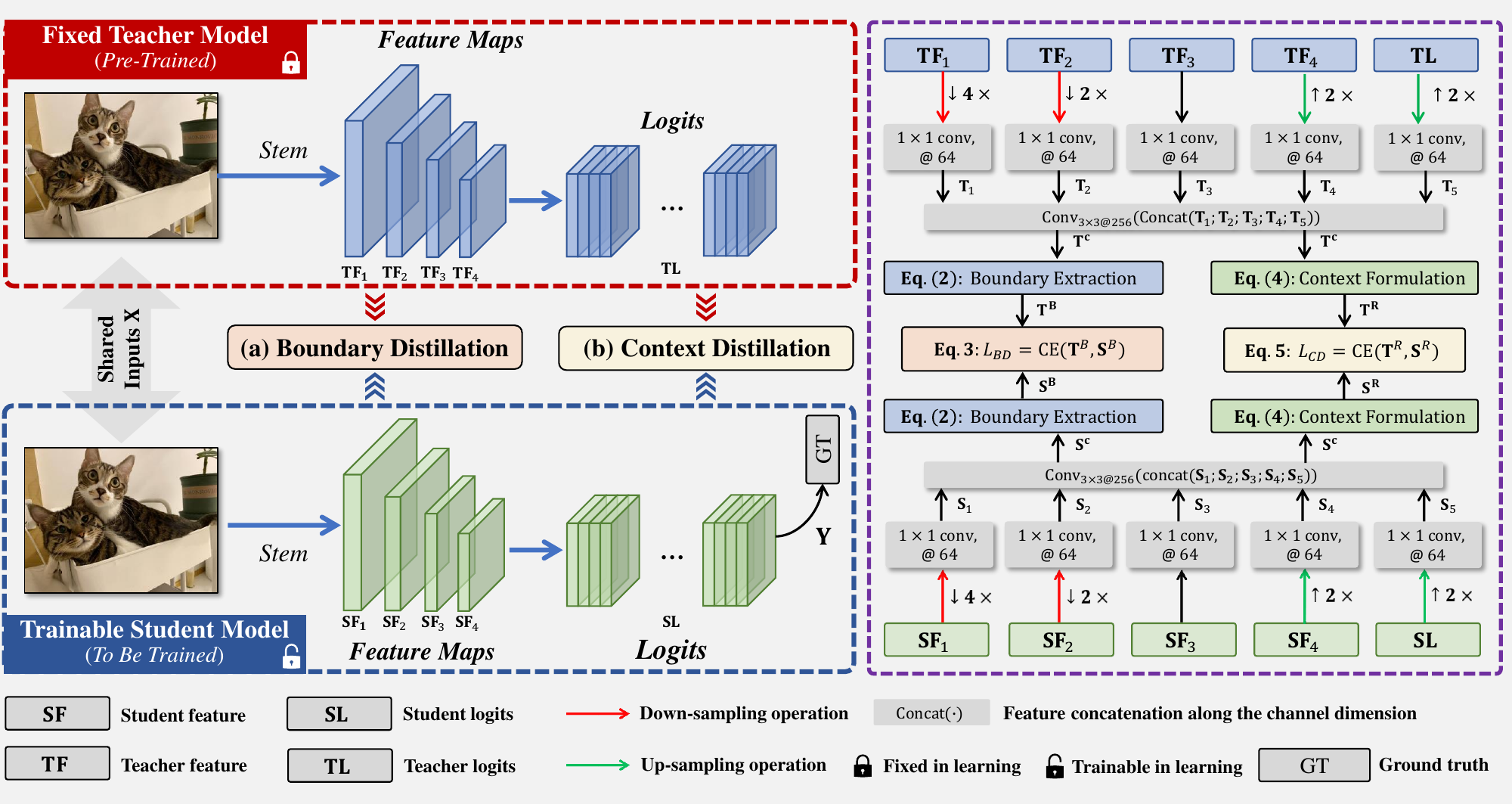}
\vspace{-4mm}
\caption{The overall network architecture of our proposed boundary and context distillation strategy for the efficient dense image prediction tasks, where the right side illustrates the implementation details of the whole network. Specifically, the {boundary distillation} involves generating explicit object-level boundaries from the hierarchical backbone features, enhancing the completeness of the student model's masks in boundary regions (\emph{ref.}~Sec.~\ref{sec:4:2}). At the same time, the {context distillation} transfers implicit pixel-level contexts from the teacher model to the student model through self-relations, ensuring robust connectivity in the student's masks (\emph{ref.}~Sec.~\ref{sec:4:3}). Compared to existing methods, our method demonstrates a stronger specificity for the EDIP tasks and inherently possesses the ability to synergistically address the common errors found in small models.}
\vspace{-2mm}
\label{fig2}
\end{figure*}

\section{Our Method}
\subsection{Overview}
\label{sec:4:1}
Figure~\ref{fig2} illustrates an overview of the network architecture for our proposed BCKD. The whole network mainly consists of an accurate teacher network $\mathbb{T}$, which is a large network that has been trained, and a small efficient network $\mathbb{S}$ that is waiting to be trained. 
The input for $\mathbb{T}$ and $\mathbb{S}$ is an arbitrary RGB image $\textbf{X}$, and the output of $\mathbb{S}$ is a semantic mask or bounding box $\textbf{Y}$ that predicts each pixel or/and object with a specific class label. 
The hierarchical features extracted from the backbone network are concatenated along the channel dimension to facilitate the extraction of EDIP-specific boundary and contextual information from $\textbf{X}$. The concatenated features with $256$ channel dimension are defined as $\textbf{T}^{c} = \textrm{Conv}_{3\times3}(\textrm{concat}(\textbf{T}_1;\textbf{T}_2;\cdot \cdot \cdot;\textbf{T}_5))$ and $\textbf{S}^{c} = \textrm{Conv}_{3\times3}(\textrm{concat}(\textbf{S}_1;\textbf{S}_2;\cdot \cdot \cdot;\textbf{S}_5))$ for $\mathbb{T}$ and $\mathbb{S}$, respectively. It should be noted that both $\textbf{T}_{i}$ and $\textbf{S}_{i}$ that are concatenated have been uniformly resized into $1/8$ of $\textbf{X}$'s spatial size via $1 \times 1$ convolution and up-/down-sampling operations. 
In training, we propose targeted \emph{boundary distillation} and \emph{context distillation} strategies that are tailored for the EDIP tasks.
Specifically, the \emph{boundary distillation} synthesizes explicit object-level boundaries $\textbf{T}^{B}$ and $\textbf{S}^{B}$ from $\textbf{T}^{c}$ and $\textbf{S}^{c}$, respectively, thereby the completeness of $\mathbb{S}$'s results in the boundary regions can be enhanced. At the same time, the \emph{context distillation} transfers implicit pixel-level relations $\textbf{T}^{R}$ and $\textbf{S}^{R}$ by using self-relations, ensuring that $\mathbb{S}$'s results have strong target region connectivity. 
\subsection{Boundary Distillation}
\label{sec:4:2}
The semantic object boundary is defined as a set of arbitrary pixel pairs from the given image, where the boundary between pairwise pixels has a value of $1$ if they belong to different classes, while if the pairwise pixels belong to the same class, the boundary between them has a value of 0~\cite{ahn2018learning}. Moreover, this attribute also exists in the hierarchical features/logits extracted by the backbone feature maps~\cite{chen2020weakly}. In our work, we use the semantic affinity similarity between arbitrary pairwise pixels from $\textbf{T}^{c}$ or $\textbf{S}^{c}$ to obtain the explicit object-level boundaries~\cite{ahn2019weakly,ru2022learning}. Concretely, for a pair of image pixels $\mathbf{T}^c_i$ and $\mathbf{T}^c_j$, $\textbf{T}_{i,j}^{B}$ can be formulated as:
\begin{equation}
\small
\begin{aligned}
\textbf{T}{i,j}^{B} = 1 - \max{p,q \in \Pi_{i,j}}{\mathcal{B} \left(\textrm{Conv}_{1 \times 1}(\mathbf{T}^c{p}), \textrm{Conv}_{1 \times 1}(\mathbf{T}^c{q})\right)},
\end{aligned}
\label{eq2}
\end{equation}
where $\mathbf{T}^c_p$ and $\mathbf{T}^c_q$ are two arbitrary pixel items from $\mathbf{T}^c$, and $\Pi_{i,j}$ denotes a set of pixel items on the line between $\mathbf{T}^c_i$ and $\mathbf{T}^c_j$. $\textrm{Conv}_{1 \times 1}$ denotes a $1 \times 1$ convolution layer that is used to compress the channel dimension, where the input channel size is $256$ and output channel size is $1$. $\mathcal{B}(\cdot)$ denotes the operation that determines the object-level image boundary values, which outputs either $1$ or $0$. $\textbf{T}^{B}$ can be obtained across the entire spatial domain, and $\textbf{S}^{B}$ can be obtained analogously. Based on $\textbf{T}^{B}$ and $\textbf{S}^{B}$, {boundary distillation} loss is formulated as:
\begin{equation}
\mathcal{L}_{BD} = - \tau^2 \sum_{i \in M} {\rho(\textbf{T}^{B}_i)^{1/\tau}\textrm{log}\left(\rho(\textbf{S}^{B}_i)^{1/\tau} \right)},
\label{eq3}
\end{equation}
where $\rho$ denotes the spatial-wise softmax normalization. By Eq.~\eqref{eq3}, $\mathbb{T}$'s accurate prediction of the object boundary region can be transferred into $\mathbb{S}$, thereby addressing its issue of maintaining boundary region completeness.

The proposed boundary distillation strategy can effectively address the limitations of EDIP models in achieving completeness in boundary region predictions. \emph{It is important to highlight that while several advanced methods, such as BGLSSeg~\cite{zhou2024boundary}, SlimSeg~\cite{xue2022slimseg}, and~BPKD~\cite{liu2024bpkd}, leverage explicit edge information for semantic segmentation, these methods necessitate the pre-extraction and integration of ground-truth masks (please refer to Table~\ref{rtab3} for further details). More importantly, the extracted edge information often lacks the semantic context of the objects and may introduce noise~\cite{yang2022focal,ji2022structural}.} In contrast, our method eliminates the need for pre-extracting image edges. By utilizing hierarchical feature maps, our method enhances the delineation of object boundaries with more comprehensive semantic information while mitigating the adverse effects of noise present in shallow features. 
As illustrated in Figure~\ref{boundary}, we show a comparison of the extracted image boundaries between our method and the ground truth edge method used in the state-of-the-art BPKD~\cite{liu2024bpkd} model. We can observe that the semantic boundaries extracted by our method can better cover the actual boundaries of the semantic objects without introducing background noise information. This innovation not only can enhances the practicality of our method for real-world applications but also streamlines the overall process, leading to reductions in both time and labor requirements.
\begin{figure}[t]
\centering
\includegraphics[width=.48\textwidth]{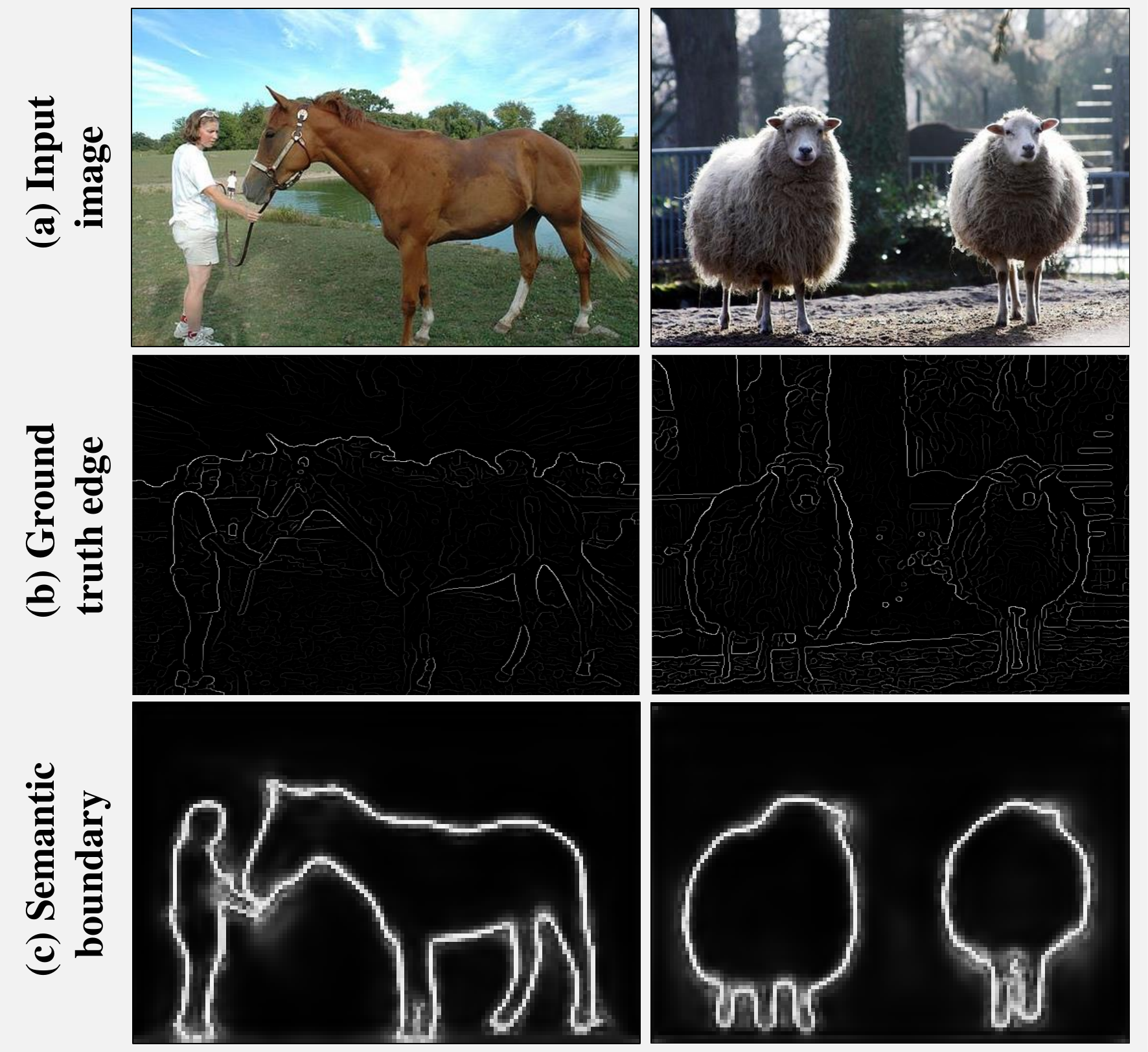}
\vspace{-2mm}
\caption{The visualization comparisons of the extracted image boundaries between our method in (c) and the ground truth edge method in (b) used in the state-of-the-art BPKD~\cite{liu2024bpkd}. Images are from Pascal VOC 2012~\cite{everingham2010pascal}.}
\vspace{-2mm}
\label{boundary}
\end{figure}

\subsection{Context Distillation}
\label{sec:4:3}
Elaborate object relations, as validated in~\cite{caron2021emerging,wang2021dense,li2022exploring}, are beneficial for learning implicit contextual information across spatial dimensions~\cite{lin2023structtoken,ji2022structural}. In this paper, we also adopt this scheme to address the challenge of inadequate preservation of target region connectivity in EDIPs. We consider this scheme as the medium in the KD process. Our contribution lies in treating pixel-level relations as a bridge for context transfer. Compared to existing methods~\cite{li2022exploring,ji2022structural}, our approach utilizes pixel-level relations solely during the training process, thereby avoiding the increase in model complexity and parameters in inference that is typically associated with current methods~\cite{liu2019structured,lin2023structtoken}. 
Besides, compared to object-level relations, the employed pixel-level relations can capture global contextual information more comprehensively, making them more suitable for DIP tasks. Specifically, for the concatenated features $\textbf{T}^{c}$ of $\mathbb{T}$, its self-relation $\textbf{T}^{R}$ is formulated as: 
\begin{equation}
\textbf{T}^{R} = \sigma\left( \frac{O(\textbf{T}^{c})^T \cdot O(\textbf{T}^{c})}{\sqrt{d}}\right)/\tau \in \mathbb{R}^{hw\times hw},
\label{eq4}
\end{equation}
where $O(\cdot)$ denotes the feature-aligned operation as in~\cite{li2022exploring,xie2023nonrigid}, which aims to align the feature distribution of the $\mathbb{S}$ with that of the $\mathbb{T}$ as closely as possible. $T$ denotes the matrix transpose operation. $d$ is the channel size of $\textbf{T}^{c}$, which is $256$. $h$ and $w$ denotes the height and width of $\textbf{T}^{c}$, respectively. $\textbf{S}^{R}$ of $\mathbb{S}$ can also be obtained analogously. Therefore, the \textsl{context distillation} can be expressed as:
\begin{equation}
\mathcal{L}_{CD} = - \tau^2 \sum_{k \in hw\times hw} {\sigma(\textbf{T}^{R}_k)^{1/\tau}\textrm{log}\left(\sigma(\textbf{S}^{R}_k)^{1/\tau} \right)},
\label{eq5}
\end{equation}
where $k = (1,2,...,hw\times hw)$ is the index item. $\textbf{T}^{R}_k$ and $\textbf{S}^{R}_k$ denotes the $k$-th item in $\textbf{T}^{R}$ and $\textbf{S}^{R}$, respectively.
\begin{figure}[t]
\centering
\includegraphics[width=.47\textwidth]{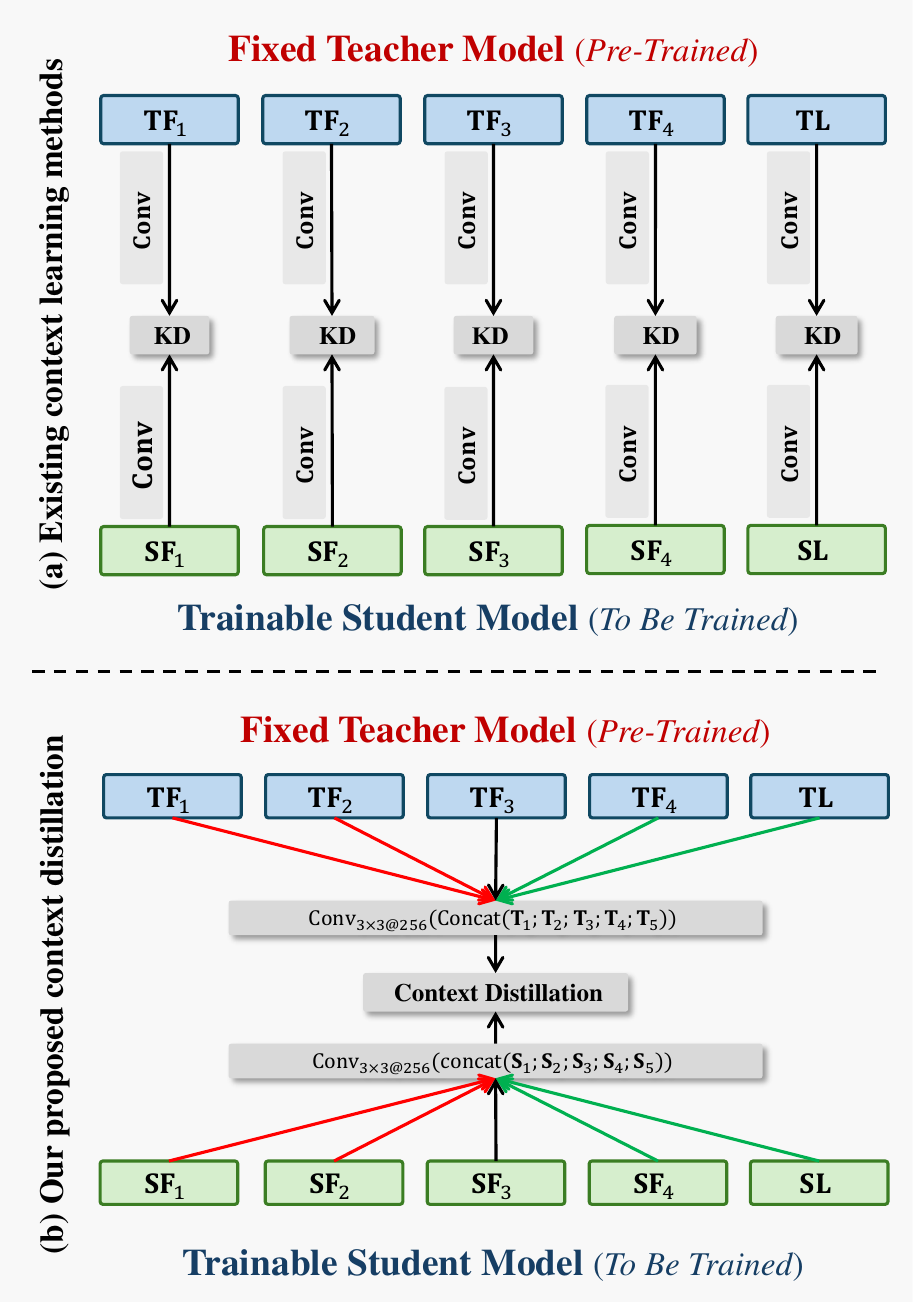}
\vspace{-3mm}
\caption{Architecture comparisons between our proposed context distillation (b) and existing context learning methods (a). Our method leverages concatenated features in a whole-to-whole manner, which avoids the potential noise introduced by layer-to-layer distillation modes.}
\vspace{-3mm}
\label{cdmanner}
\end{figure}

The proposed context distillation method presents an efficient solution that does not incur additional inference overhead. \emph{As illustrated in Figure~\ref{cdmanner}, our method in (b), which utilizes concatenated features in a whole-to-whole manner, circumvents the potential noise associated with existing context learning methods in (a) for semantic segmentation~\cite{liu2019structured,liu2024bpkd} that rely on layer-to-layer distillation.} Our method is specifically tailored to address the potential challenge of incomplete preservation of target region connectivity as shown in Figure~\ref{fig1}, rather than solely focusing on the enhancement of feature representations in a generic context.

\subsection{Overall Loss Function}
\label{sec:4:4}
With the {boundary distillation} loss $\mathcal{L}_{BD}$ and the {context distillation} loss $\mathcal{L}_{CD}$, the total loss can be expressed as:
\begin{equation}
\mathcal{L} = \mathcal{L}_{SS} + \alpha \mathcal{L}_{BD} + \beta \mathcal{L}_{CD},
\label{eq6}
\end{equation}
where $\alpha$ and $\beta$ are two weights used to balance different losses. Empirically, these weights have a significant impact on model performance, and unreasonable weight settings may even lead to model collapse in training. To this end, inspired by previous work~\cite{sun2020distilling,yang2023effective}, to enhance the dependence of $\mathbb{S}$ on ground truth labels, we incorporate a weight-decay strategy that promotes a greater focus on $\mathcal{L}_{SS}$ as the training epoch increases. To this end, we initialize a time function as follows:
\begin{equation}
r(t) = 1 - (t-1)/t_{\textrm{max}},
\label{eq7}
\end{equation}
where $t = (1,2,..., t_{\textrm{max}})$ denotes the current training epoch and $t_{\textrm{max}}$ is the maximum training epoch. In training, we control the dependence of $\mathcal{L}$ on $\mathcal{L}_{BD}$ and $\mathcal{L}_{CD}$ by using $r(t)$. Therefore, the total loss in Eq.~\eqref{eq6} can be formulated as:
\begin{equation}
\mathcal{L} = \mathcal{L}_{SS} + r(t) \alpha \mathcal{L}_{BD} + r(t) \beta \mathcal{L}_{CD}.
\label{eq8}
\end{equation}
\subsection{Theoretical Analysis}
\label{sec:4:5}
Our BCKD is theoretically grounded in differential geometry and spectral graph theory. In the following, we will analyze its components through measurable properties of the learned feature manifolds $\mathcal{M}_T$ (teacher) and $\mathcal{M}_S$ (student), with proofs connecting to the empirical results in Section~\ref{ablation}.

\myparagraph{Boundary-Aware Manifold Alignment.}
The proposed $\mathcal{L}_{BD}$ in Eq.~\eqref{eq3} induces geometric consistency between teacher and student decision boundaries. 

\begin{theorem}[Boundary Consistency]
Under $\mathcal{L}_{BD}$ minimization with $\tau > 1$, for any boundary point $x_b$, we have:
\begin{equation}
\|\mathbf{J}_T(x_b) - \mathbf{J}_S(x_b)\|_F \leq \sqrt{2\mathcal{L}_{BD}/\tau^2} + \mathcal{O}(e^{-\tau}),
\end{equation}
where $\mathbf{J}_\cdot$ are Jacobian matrices of the feature maps.
\end{theorem}

\begin{proof}
The temperature-scaled gradients satisfy:
\begin{align}
\nabla\mathcal{L}_{BD} &= \mathbb{E}_{x_b}\left[\frac{\sigma_T^{1/\tau}}{\sigma_S^{1/\tau}}\nabla\log\sigma_S\right], 
\label{eq:grad_bound}\\
\|\nabla L_T - \nabla L_S\|^2 &\leq 2(1 - \cos\theta),
\label{eq:angle_bound}
\end{align}
where $\theta$ is the angle between $\mathbb{S}$ and $\mathbb{T}$ gradients. Applying Taylor expansion~\cite{kanwal1989taylor} at high $\tau$, then we can obtain:
\begin{equation}
\cos\theta \geq 1 - \frac{1}{2}\mathcal{L}_{BD}\tau^{-2} + \mathcal{O}(\tau^{-4}).
\end{equation}
The Jacobian bound follows from Pinsker's inequality applied to the manifold tangent spaces.
\end{proof}
This theoretical guarantee explains the $0.91\%$ mIoU improvement observed in Table~\ref{rtab1}, as aligned Jacobians ensure consistent boundary localization.

\myparagraph{Contextual Graph Preservation}
The proposed $\mathcal{L}_{CD}$ in Eq.~\eqref{eq5} maintains spectral properties critical for dense image prediction tasks:

\begin{theorem}[Spectral Convergence]
For eigenvalues $\{\lambda_i\}$ of relation matrices $\mathbf{T}^R,\mathbf{S}^R$, we have:
\begin{equation}
\max_i|\lambda_i^T - \lambda_i^S| \leq \|\mathbf{T}^R - \mathbf{S}^R\|_F \leq \sqrt{d}\mathcal{L}_{CD}.
\end{equation}
\end{theorem}

\begin{proof}
Applying \textit{Weyl's inequality} for symmetric matrices, we can obtain:
\begin{equation}
    |\lambda_i^T - \lambda_i^S| \leq \|\Delta\mathbf{R}\|_2 \leq \|\Delta\mathbf{R}\|_F,
    \end{equation}
where the heat kernel continuity follows from:
    \begin{align}
    \|e^{-\tau\mathcal{L}_T} - e^{-\tau\mathcal{L}_S}\|_F &\leq \tau \sup_{t\in[0,\tau]}\|e^{-t\mathcal{L}_T}(\mathcal{L}_T - \mathcal{L}_S)e^{-(\tau-t)\mathcal{L}_S}\| \\
    &\leq \tau e^{\tau\|\mathcal{L}_T\|} \mathcal{L}_{CD} \label{eq:heat_bound}
    \end{align}
\end{proof}
As shown in Table~\ref{rtab1}, compared with the baseline model, the $1.77\%$ mIoU gain directly reflects this eigenvalue stability.

\myparagraph{Multi-Scale Geometric Consistency}
As illustrated in Figure~\ref{cdmanner}, the feature concatenation and projection operation in Section~\ref{sec:4:1} preserves topological invariants:

\begin{proposition}[Topological Preservation]
The mapping $\phi:\prod_i\mathcal{M}_T^{(i)}\to\mathcal{M}_T^{\text{concat}}$ satisfies:
\begin{equation}
\beta_k(\mathcal{M}_T^{\text{concat}}) = \sum_{i=1}^5 \beta_k(\mathcal{M}_T^{(i)}), \quad k=0,1,2
\end{equation}
where $\beta_k$ are Betti numbers.
\end{proposition}

\begin{proof}
The 3$\times$3 convolution operation is a diffeomorphism, thus we have:
\begin{align}
\beta_k(\phi(\mathbf{T}^c)) &= \beta_k(\mathbf{T}^c) \quad \text{(invariance)} \\
&= \beta_k(\oplus_i \mathcal{M}_T^{(i)}) \nonumber \\
&= \sum_{i=1}^5 \beta_k(\mathcal{M}_T^{(i)}) \quad \text{(Künneth formula)} \label{eq:topo_preserve}
\end{align}
The dimensionality bound follows from the classical projection theorem as in~\cite{falconer1996projection}.
\end{proof}

\myparagraph{Training Dynamics Interpretation}
The used weight decay in Eq.~\eqref{eq8} induces phased learning:

\begin{theorem}[Annealed Convergence]
With $r(t) = 1 - (t-1)/t_{\max}$ and Robbins-Monro conditions on learning rate $\eta_t$:
\begin{equation}
\lim_{t\to t_{\max}} \mathbb{P}(\mathcal{L} = \mathcal{L}_{SS}) = 1
\end{equation}
\end{theorem}

\begin{proof}
Decompose the gradient flow:
\begin{align}
\frac{d\mathcal{L}}{dt} &= -\eta_t\|\nabla\mathcal{L}_{SS}\|^2 \nonumber \\
&\quad - \eta_tr(t)^2(\alpha^2\|\nabla\mathcal{L}_{BD}\|^2 + \beta^2\|\nabla\mathcal{L}_{CD}\|^2) \label{eq:grad_flow}
\end{align}
As $r(t)\to 0$, the right terms vanish asymptotically. The convergence follows the stochastic approximation theory~\cite{lai2003stochastic}.
\end{proof}

\section{Experiments}
\subsection{Datasets and Evaluation Metrics}
\subsubsection{Datasets} To demonstrate the superior performance of our method, we conduct experiments on five representative yet challenging datasets: Pascal VOC 2012 \cite{everingham2010pascal}, Cityscapes \cite{cordts2016cityscapes}, ADE20K \cite{zhou2017scene}, and COCO-Stuff 10K \cite{caesar2018coco} for semantic segmentation (SSeg), as well as MS-COCO 2017 \cite{lin2014microsoft} for instance segmentation (ISeg) and object detection (ODet).
\begin{itemize}
\item The Pascal VOC 2012 dataset comprises $20$ object classes along with one background class. Following~\cite{zhang2021self,wang2020intra}, we utilized the augmented data, resulting in a total of $10,582$ images for \emph{training}, $1,449$ images for \emph{val}, and $1,456$ images for \emph{testing}. 
\item The Cityscapes dataset comprises a total of $5,000$ finely annotated images, which are partitioned into subsets of $2,975$, $500$, and $1,525$ images designated for \emph{training}, \emph{val}, and \emph{testing}, respectively. In alignment with existing methods~\cite{zheng2023distilling,wang2020intra,liu2019structured}, we exclusively employed the finely labeled data during the training phase to ensure a fair comparison of results.
\item The ADE20K dataset has $150$ object classes and is organized into three subsets: $20,000$ images for the \emph{training} set, $2,000$ images for the \emph{val} set, and $3,000$ images for the \emph{testing} set.
\item The COCO-Stuff 10K dataset is an extension of the MS-COCO dataset \cite{lin2014microsoft}, enriched with pixel-wise class labels. It consists of $9,000$ samples designated for \emph{training} and $1,000$ samples allocated for \emph{val}.
\item The MS-COCO 2017 dataset comprises $80$ object classes and includes a total of $118,000$ images for \emph{training}, and $5,000$ images for \emph{val}.
\end{itemize}
For data augmentation, random horizontal flip, brightness jittering and random scaling within the range of $[0.5, 2]$ are used in training as in~\cite{wang2020intra,zhang2021self,liu2019structured}. Experiments are implemented on the MMRazor framework\footnote{https://github.com/open-mmlab/mmrazor} under the PyTorch platform~\cite{paszke2019pytorch} using $8$ NVIDIA GeForce RTX 3090 GPUs. All the inference results are obtained at a single scale. 

\subsubsection{Evaluation metrics} 
Beyond employing standard metrics, we have also developed two extra specialized evaluation metrics specifically optimized for knowledge distillation on dense image prediction tasks, detailed below:

\myparagraph{Common metrics.} For SSeg, we utilize the mean intersection over union (mIoU) as the primary evaluation metric. For ISeg and ODet, average precision (AP) serves as the principal accuracy-specific metric. To assess model efficiency, we also consider the number of parameters (Params.) and the floating-point operations (FLOPs). 

\myparagraph{Manifold stability (MFS).} To assess the effectiveness on the learned feature manifolds, we compute the Lipschitz constant ratio between teacher and student models. Specifically, let $f_T^l(x), f_S^l(x) \in \mathbb{R}^{d_l}$ denote the $l$-th layer feature mappings for teacher and student models respectively. The layer-wise Lipschitz constant is estimated via:
\begin{equation}
\small
L_k^l = \sup_{x \in \mathcal{X}, \|\delta\| \leq \epsilon} \frac{\|f_k^l(x+\delta) - f_k^l(x)\|_2}{\|\delta\|_2}, \quad k \in \{T,S\}
\end{equation}
where $\mathcal{X}$ is the input space and $\epsilon = 0.1$ controls the perturbation scale. The MFS metric $\rho_l$ is then computed as:
\begin{equation}
\small
\rho_l = \frac{L_S^l}{L_T^l} \cdot \mathbb{I}(L_T^l > \tau) + \mathbb{I}(L_T^l \leq \tau)
\end{equation}
with threshold $\tau$ avoiding division by negligible values ($\tau=0.01$). Values close to 1 indicate well-preserved manifold structure during distillation, while significant deviations suggest potential degradation of geometric properties.

\myparagraph{Local Hausdorff distance (LHD).} For boundary-sensitive tasks like SSeg and ISeg, we introduce a LHD measure to evaluate boundary alignment quality. Specifically, for boundary point sets $\mathcal{B}_p = \{p_i\}_{i=1}^m$ and $\mathcal{B}_g = \{q_j\}_{j=1}^n$, 
the LHD at point $p_i$ is defined as:
\begin{equation}
\small
\begin{aligned}
\text{LHD}_r(p_i, \mathcal{B}_g) 
= \min\bigg\{ 
    &\max_{q_j \in N_r(p_i)} d(p_i,q_j), \\ 
    &\text{median}\big(\{d(p_i,q_j)\}_{q_j \in N_r(p_i)}\big) 
\bigg\},
\end{aligned}
\end{equation}
where the neighborhood $N_r(p_i)$ and final aggregation are:
\begin{equation}
\small
\begin{aligned}
N_r(p_i) &= \{q_j \in \mathcal{B}_g \mid \|p_i - q_j\|_2 \leq r\} \\
\text{LHD}(\mathcal{B}_p, \mathcal{B}_g) &= 
\frac{1}{|\mathcal{B}_p|}\sum_{i=1}^{|\mathcal{B}_p|} \text{LHD}_r(p_i, \mathcal{B}_g) \\ 
    &\cdot \mathbb{I}\big(\text{LHD}_r(p_i, \mathcal{B}_g) \leq \mu + 2\sigma\big),
\end{aligned}
\end{equation}
where $r$ is set to 5 in our implementation.
\subsection{Implementation Details}

\subsubsection{Baselines} For a fair result comparison and considering the realistic resource conditions of edge computing devices~\cite{dong20231920}, \emph{\textbf{for SSeg}}, we select \textbf{PSPNet-101}~\cite{zhao2017pyramid}, \textbf{DeepLabv3 Plus-101}~\cite{chen2018encoder}, and \textbf{Mask2Former}~\cite{cheng2022masked} for the teacher models. The student models are compact \textbf{PSPNet} and \textbf{DeepLabV3+} with ResNet-38, ResNet-18$_{(0.5)}$ and ResNet-18$_{(1.0)}$~\cite{he2016deep}. Besides, to demonstrate the superior effectiveness of our method on heterogeneous network architectures, following~\cite{wang2020intra,liu2019structured,yang2022cross}, we also employ \textbf{MobileNetV2}~\cite{liu2018lightnet}, \textbf{EfficientNet-B1}~\cite{tan2019efficientnet}, and \textbf{SegFormer-B0}~\cite{xie2021segformer} as the student models. 
\emph{\textbf{For ISeg and ODet}}, following~\cite{wang2024crosskd,zhang2023structured}, we select the representative \textbf{GFL}~\cite{li2020generalized}, \textbf{Cascade Mask R-CNN}~\cite{cai2018cascade}, and \textbf{RetinaNet}~\cite{ross2017focal} with ResNet-101/50 and ResNet-50/18 as the teacher model and the student model, respectively. 
\emph{While adopting cutting-edge large foundation models represents the prevailing research trend, we opt for a pragmatic small-model baseline given the currently insurmountable challenges in deploying such large-scale models on edge devices. This application-oriented approach prioritizes practical deployability over model scale, while still providing a meaningful benchmark for edge computing scenarios.} During the training and inference phases, aside from our proposed method and specific declarations, all other settings adhere to the configurations outlined in the baseline model.

\subsubsection{Training details} Following the standard practices as in~\cite{hinton2015distilling,liu2019structured,wang2020intra}, all teacher models are pre-trained on ImageNet-1k by default~\cite{deng2009imagenet}, and then fine-tuned on the corresponding dataset before their parameters are fixed in KD. During training, only the student's parameters are updated. As in~\cite{zhang2021self,phuong2019distillation}, $\tau$ is initialized to $1$ and is multiplied by a scaling factor of $1.05$ whenever the range of values (across all feature items in a given minibatch) exceeds $0.5$. Following~\cite{wang2020intra,sun2020distilling,yang2023effective}, $\alpha$ and $\beta$ is set to 10 and 50, respectively. We fully understand that fine-tuning these base hyperparameters could potentially enhance performance. However, we contend that such adjustments may be redundant and unwarranted.

\myparagraph{For SSeg models}, to accommodate the local hardware limitations, the training images are cropped into a fixed size of 512 $\times$ 512 pixels as in~\cite{wang2024crosskd,zhang2023structured}. The SGD is used as the optimizer with the ``poly'' learning rate strategy. The initial learning rate is set to $0.01$, with a power of $0.9$. To ensure fairness in the experimental comparisons, the batch size is set to $8$ and $t_{\textrm{max}}$ is set to $40,000$. 

\myparagraph{For ISeg and ODet}, the model is trained following the default 1$\times$ training schedule, \ie, $12$ epochs. The batch size is set to $8$, and AdamW is used as the optimizer with the initial learning rate of $1 \times 10^{-4}$ and the weight decay of $0.05$. The layer-wise learning rate decay is used and set to $0.9$, and the drop path rate is set to $0.4$. The given images are resized to the shorter side of $800$ pixels, with the longer side not exceeding $1,333$ pixels. In inference, the shorter side of images is consistently set to 800 pixels by default. 
\begin{table}[t]
\centering
\renewcommand\arraystretch{1.2}
\setlength{\tabcolsep}{9pt}{
\caption{Ablation results on the \emph{val} set of Pascal VOC 2012~\cite{everingham2010pascal} by adding each component of BCKD. ``WD'' denotes the weight-decay strategy as introduced in Section~\ref{sec:4:4}}
\label{rtab1}
\begin{tabular}{ccc|ccc}
\hline \hline 
\multicolumn{3}{r}{$\mathbb{T}$: PSPNet-101~\cite{zhao2017pyramid}} & 77.82 \% & 70.43\textbf{M} & 411.6\textbf{G} \\
\multicolumn{3}{r}{$\mathbb{S}$: PSPNet-18$_{(1.0)}$~\cite{zhao2017pyramid}} & 70.78 \% & 15.24\textbf{M} & 106.2\textbf{G} \\
\hline 
$\mathcal{L}_{BD}$ & $\mathcal{L}_{CD}$ & WD & mIoU (\%) & Params. & FLOPs\\
\hline 
\cmark & \xmark & \xmark & 71.69$_{\color{red}{+0.91}}$ & 15.24\textbf{M} & 106.2\textbf{G}\\ 
\xmark & \cmark & \xmark & 72.55$_{\color{red}{+1.77}}$ & 15.24\textbf{M} & 106.2\textbf{G} \\ 
\cmark & \cmark & \xmark & 73.98$_{\color{red}{+3.20}}$ & 15.24\textbf{M} & 106.2\textbf{G} \\ 
\cmark & \cmark & \cmark & 74.65$_{\color{red}{+3.87}}$ & 15.24\textbf{M} & 106.2\textbf{G} \\ 
\hline \hline 
\end{tabular}}
\vspace{-3mm}
\end{table}
\begin{figure}[t]
\centering
\includegraphics[width=.48\textwidth]{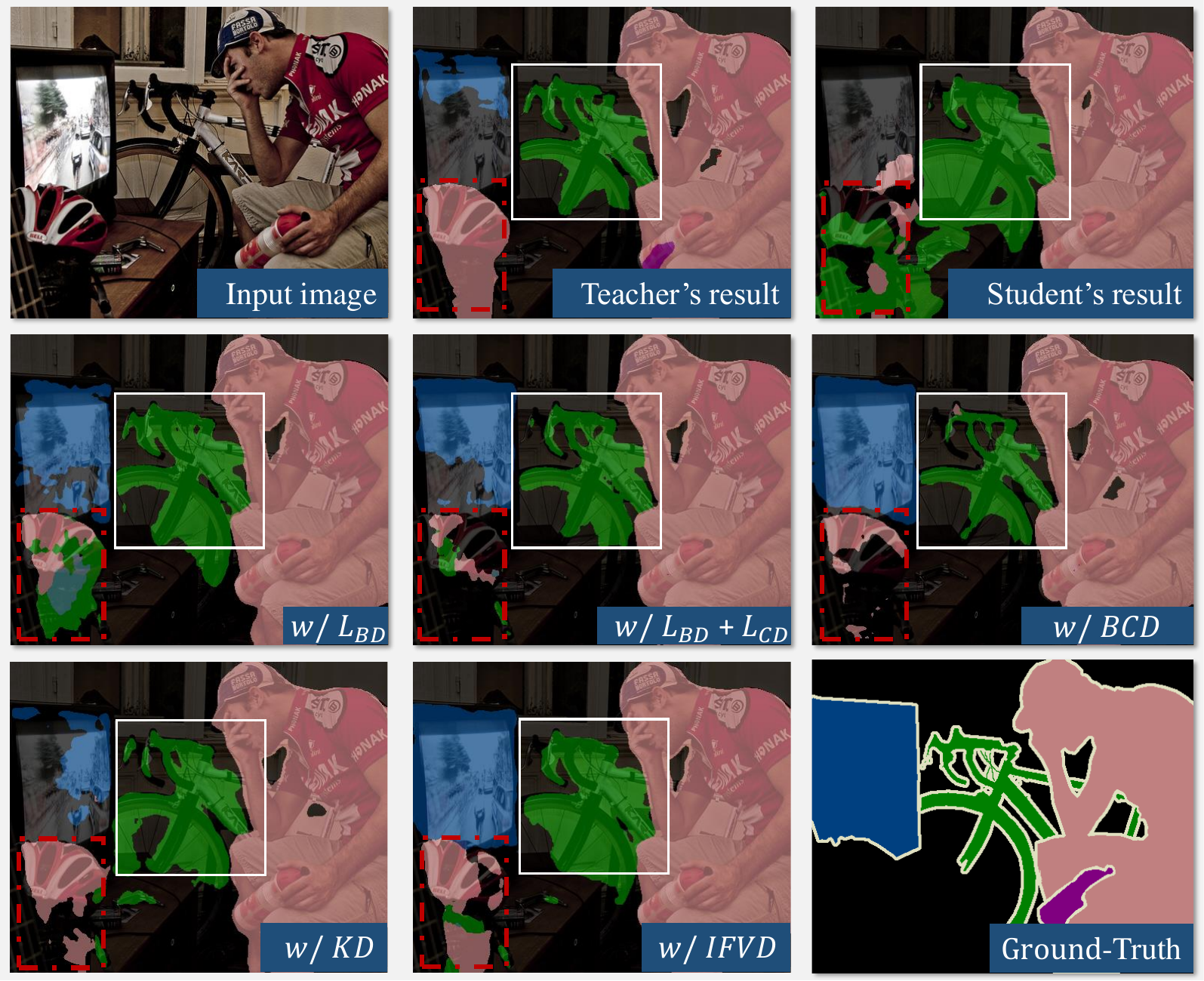}
\vspace{-3mm}
\caption{Visualized comparisons and results obtained by adding different components of BCKD. The teacher model and the student model denotes PSPNet-101~\cite{zhao2017pyramid} and PSPNet-18$_{(1.0)}$~\cite{zhao2017pyramid}, respectively. ``\emph{w/}'' denotes with the corresponding implementation. Samples are from Pascal VOC 2012~\cite{everingham2010pascal}.}
\vspace{-4mm}
\label{fig4}
\end{figure}
\begin{table}[htb]
\centering
\renewcommand\arraystretch{1.2}
\setlength{\tabcolsep}{10pt}{
\caption{Result comparisons on mIoU (\%) under different network architectures on the \emph{val} sets of Pascal VOC 2012~\cite{everingham2010pascal}, Cityscapes~\cite{cordts2016cityscapes}, and ADE20K~\cite{zhou2017scene}, and on the \emph{test} set of COCO-Stuff 10K~\cite{caesar2018coco}.}
\label{rtable2}
\begin{tabular}{r|lr}
\hline \hline 
Methods & Pascal VOC \emph{val} & Params. \\
\hline 
DANet~\cite{fu2019dual} & 80.40 & 83.1\textbf{M}\\
\hline 
$\mathbb{T}$: PSPNet-101~\cite{zhao2017pyramid} & 77.82 & 70.4\textbf{M}\\
\hline 
$\mathbb{S}$: PSPNet-38~\cite{zhao2017pyramid} & 72.65 & 58.6\textbf{M}\\
\cellcolor[gray]{.95}\textbf{+ BCKD$_{\textrm{ours}}$} & \cellcolor[gray]{.95}75.15$_{\color{red}{+2.50}}$ & \cellcolor[gray]{.95}58.6\textbf{M} \\ 
$\mathbb{S}$: EfficientNet-B1~\cite{tan2019efficientnet} & 69.28 & 6.7\textbf{M} \\
\cellcolor[gray]{.95}\textbf{+ BCKD$_{\textrm{ours}}$} & \cellcolor[gray]{.95}73.81$_{\color{red}{+4.53}}$  & \cellcolor[gray]{.95}6.7\textbf{M} \\ 
$\mathbb{S}$: SegFormer-B0~\cite{xie2021segformer} & 66.75$^\ddag$ & 3.8\textbf{M}\\
\cellcolor[gray]{.95}\textbf{+ BCKD$_{\textrm{ours}}$} & \cellcolor[gray]{.95}68.88$_{\color{red}{+2.13}}$  & \cellcolor[gray]{.95}3.8\textbf{M} \\
\hline \hline 
Methods & Cityscapes \emph{val} & Params. \\
\hline 
HRNet~\cite{sun2019high} & 81.10 & 65.9\textbf{M}\\
\hline 
$\mathbb{T}$: PSPNet-101~\cite{zhao2017pyramid} & 78.56 & 70.4\textbf{M}\\
\hline 
$\mathbb{S}$: PSPNet-38~\cite{zhao2017pyramid} & 71.26 & 47.4\textbf{M}\\
\cellcolor[gray]{.95}\textbf{+ BCKD$_{\textrm{ours}}$} & \cellcolor[gray]{.95}73.56$_{\color{red}{+2.30}}$ & \cellcolor[gray]{.95}47.4\textbf{M}\\
$\mathbb{S}$: EfficientNet-B1~\cite{tan2019efficientnet} & 60.40 & 6.7\textbf{M}\\
\cellcolor[gray]{.95}\textbf{+ BCKD$_{\textrm{ours}}$} & \cellcolor[gray]{.95}63.91$_{\color{red}{+3.51}}$ & \cellcolor[gray]{.95}6.7\textbf{M} \\ 
$\mathbb{S}$: SegFormer-B0~\cite{xie2021segformer} & 76.20 & 3.8\textbf{M}\\
\cellcolor[gray]{.95}\textbf{+ BCKD$_{\textrm{ours}}$} & \cellcolor[gray]{.95}77.87$_{\color{red}{+1.67}}$ & \cellcolor[gray]{.95}3.8\textbf{M}\\
\hline \hline 
Methods & ADE20K \emph{val} & Params.\\
\hline 
DeepLab V3~\cite{chen2017deeplab} & 43.28 & 71.3\textbf{M} \\
\hline 
$\mathbb{T}$: PSPNet-101~\cite{zhao2017pyramid} & 42.19 & 70.4\textbf{M}\\
\hline 
$\mathbb{S}$: ESPNet~\cite{mehta2018espnet} & 20.13 & 0.4\textbf{M}\\
\cellcolor[gray]{.95}\textbf{+ BCKD$_{\textrm{ours}}$} & \cellcolor[gray]{.95}23.65$_{\color{red}{+3.52}}$ & \cellcolor[gray]{.95}0.4\textbf{M}\\ 
$\mathbb{S}$: MobileNetV2~\cite{liu2018lightnet} & 33.64 & 8.3\textbf{M} \\
\cellcolor[gray]{.95}\textbf{+ BCKD$_{\textrm{ours}}$} & \cellcolor[gray]{.95}36.59$_{\color{red}{+2.95}}$  & \cellcolor[gray]{.95}8.3\textbf{M} \\ 
$\mathbb{S}$: SegFormer-B0~\cite{xie2021segformer} & 37.40 & 3.8\textbf{M}\\
\cellcolor[gray]{.95}\textbf{+ BCKD$_{\textrm{ours}}$} & \cellcolor[gray]{.95}38.75$_{\color{red}{+1.35}}$ & \cellcolor[gray]{.95}3.8\textbf{M} \\
\hline 
$\mathbb{T}$: Mask2Former~\cite{cheng2022masked} & 47.20 & 44.0\textbf{M} \\
\hline 
\cellcolor[gray]{.95}$\mathbb{S}$: \textbf{ESPNet + BCKD$_{\textrm{ours}}$} & \cellcolor[gray]{.95}24.26$_{\color{red}{+4.13}}$ & \cellcolor[gray]{.95}0.4\textbf{M} \\
\cellcolor[gray]{.95}$\mathbb{S}$: \textbf{MobileNetV2 + BCKD$_{\textrm{ours}}$} & \cellcolor[gray]{.95}37.09$_{\color{red}{+3.45}}$ & \cellcolor[gray]{.95}8.3\textbf{M} \\
\cellcolor[gray]{.95}$\mathbb{S}$: \textbf{SegFormer-B0 + BCKD$_{\textrm{ours}}$} & \cellcolor[gray]{.95}39.61$_{\color{red}{+2.21}}$ & \cellcolor[gray]{.95}3.8\textbf{M}\\
\hline \hline 
Methods & COCO 10K \emph{test} & FLOPs\\
\hline 
SegVIT~\cite{zhang2022segvit} & 50.3 & 383.9\textbf{G}\\
\hline 
$\mathbb{T}$: DeepLabV3 Plus-101~\cite{chen2018encoder} & 33.10 & 366.9\textbf{G}\\
\hline 
$\mathbb{S}$: MobileNetV2~\cite{liu2018lightnet} & 26.29 & 1.4\textbf{G}\\
\cellcolor[gray]{.95}\textbf{+ BCKD$_{\textrm{ours}}$} & \cellcolor[gray]{.95}28.31$_{\color{red}{+2.02}}$ & \cellcolor[gray]{.95}1.4\textbf{G} \\ 
$\mathbb{S}$: SegFormer-B0~\cite{xie2021segformer} & 35.60 & 8.4\textbf{G}\\
\cellcolor[gray]{.95}\textbf{+ BCKD$_{\textrm{ours}}$} & \cellcolor[gray]{.95}35.92$_{\color{red}{+0.32}}$ & \cellcolor[gray]{.95}8.4\textbf{G} \\
\hline \hline 
\end{tabular}}
\vspace{-6mm}
\end{table}

\subsection{Ablation Analysis}
\label{ablation}
In our ablation analysis, we aim to explore answers of the following crucial questions: \emph{1)}~the impact of each component within BCKD; \emph{2)}~the effectiveness of BCKD across different network architectures; and \emph{3)}~the visualized performance and comparisons with other methods. We select the SSeg task as the experimental objective. 
\subsubsection{Effectiveness of each component} 
To explore answer of the first question, we chose Pascal VOC 2012 \cite{everingham2010pascal} as the experimental dataset. PSPNet-101~\cite{zhao2017pyramid} serves as the teacher model, while PSPNet-18$_{(1.0)}$ is used as the student model.
Table~\ref{rtab1} shows the inference results by adding each component of BCKD into the student model, where we report the experimental results on the \emph{val} set. We can observe that incorporating these components consistently improves the model accuracy, indicating the effectiveness of these components. In particular, adding $\mathcal{L}_{BD}$ resulted in a mIoU$\uparrow$ gain of $0.91\%$, which may be attributed to the fact that the object boundary regions are relatively small in proportion to the entire image~\cite{zhang2021self,liu2024bpkd}. With only the implementation of $\mathcal{L}_{BD}$ and $\mathcal{L}_{CD}$, our model can achieve the competitive 73.98\% mIoU with $3.20\%$ mIoU$\uparrow$ improvements, which verifies the importance of boundary and context information in EDIP. Furthermore, this result demonstrates that our$\mathcal{L}_{BD}$ and $\mathcal{L}_{CD}$ do not conflict in practical deployment; instead, they complement each other intrinsically to enhance performance.
It is also worth noting that compared to the advanced methods in Table~\ref{rtab3} that do not utilize ground truth mask, our method also achieves competitive results even without employing the weight-decay strategy. Building upon $\mathcal{L}_{BD}$ and $\mathcal{L}_{CD}$, adding the weight-decay strategy brings $0.57\%$ mIoU$\uparrow$, which confirms the weight importance of different losses. Furthermore, there is no increase in the number of Params. or FLOPs in the inference stage.
\begin{figure*}[htb]
\centering
\includegraphics[width=1\textwidth]{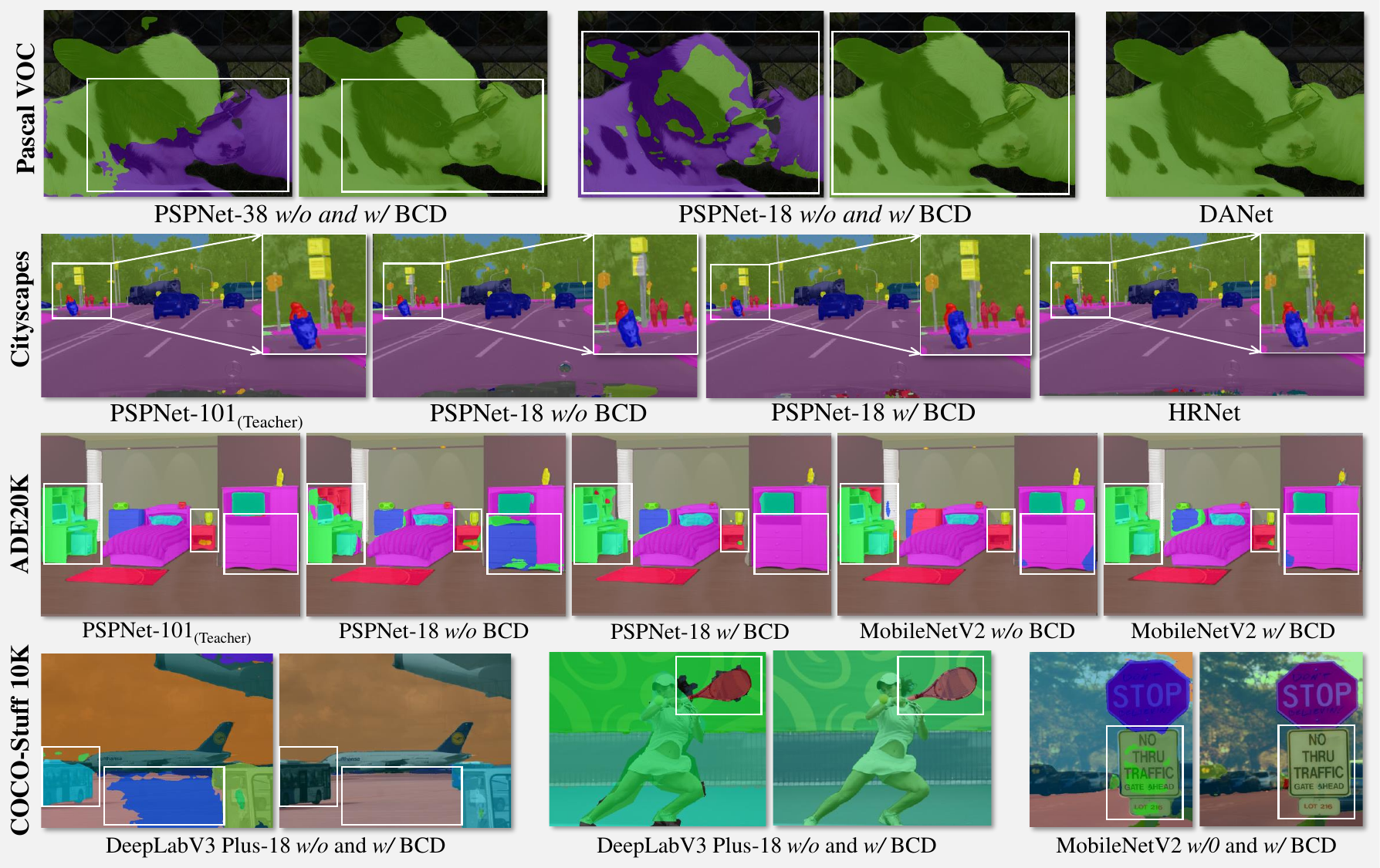}
\vspace{-6mm}
\caption{Visualizations on SSeg. DANet~\cite{fu2019dual} and HRNet~\cite{sun2019high} have been included for comparison as well. ``\textsl{w/o}'' means ``without'' and ``\textsl{w/}'' means ``with'', indicating whether our method is NOT implemented or implemented. The white bounding boxes highlight the regions where our method predicts better.}
\vspace{-4mm}
\label{fig3}
\end{figure*}


We also employ visualizations as a means of verifying the efficacy of BCKD components incrementally into the baseline model. The obtained results are presented in Figure~\ref{fig4}, which demonstrate that the inclusion of $\mathcal{L}_{BD}$ and $\mathcal{L}_{CD}$ in a sequential manner leads to enhanced boundary regions and object connectivity. For example, the ``\emph{bike handlebar}''. Moreover, the integration of the weight-decay strategy results in further improvement in the overall segmentation predictions. 
In addition to using white bounding boxes to emphasize the better regions achieved by our method, we also highlighted the regions of prediction failure using \textcolor{red}{red dashed bounding boxes}. We can observe that although our BCKD significantly improves the prediction quality of these regions compared to the student model’s results, there are still some incomplete predictions. This case may be caused by the spurious correlation between the ``\emph{helmet}'' and the ``\emph{person}'' in the used dataset, which can be eliminated through causal intervention~\cite{zhang2020causal}.

\subsubsection{Effectiveness across network architectures} 
In this section, we evaluate the effectiveness of our BCKD across various network architectures for SSeg using: Pascal VOC 2012~\cite{everingham2010pascal}, Cityscapes~\cite{cordts2016cityscapes}, ADE20K~\cite{zhou2017scene}, and COCO-Stuff 10K~\cite{caesar2018coco}. The obtained results under different network architectures are given in Table~\ref{rtable2}. 
For the purpose of comparing experimental results, we also include the results of a large model that does not utilize knowledge distillation for each dataset.
From this table, we can observe that deploying BCKD on different network architectures can lead to continuous performance improvements. For example, when employing PSPNet-38~\cite{zhao2017pyramid}, EfficientNet-B1~\cite{tan2019efficientnet}, and SegFormer-B0~\cite{xie2021segformer} as student models while keeping the teacher model PSPNet-101~\cite{zhao2017pyramid} unchanged, our BCKD achieves mIoU$\uparrow$ improvements of $2.50\%$, $4.53\%$, and $2.13\%$ on Pascal VOC \emph{val}, and $2.30\%$, $3.51\%$, and $1.67\%$ on Cityscapes \emph{val}, respectively. The performance gains demonstrate the effectiveness of our method not only within the same network architecture but also across network architectures, indicating sustained performance enhancements. This phenomenon also highlights the generalization capacity of our method. Besides, similar conclusions can be also drawn from our experimental results on ADE20K \emph{val} and COCO-Stuff 10K \emph{test} sets. Deploying different teacher models on the ADE20K \emph{val} dataset are also conducted, where Mask2Former~\cite{cheng2022masked} is utilized as the teacher model, and ESPNet~\cite{mehta2018espnet}, MobileNetV2~\cite{liu2018lightnet}, and SegFormer-B0 are used as the student models. The results demonstrate that our method yields mIoU$\uparrow$ improvements of $4.13\%$, $3.45\%$, and $2.21\%$ on ESPNet, MobileNetV2, and SegFormer-B0~\cite{xie2021segformer}, respectively, showcasing the strong flexibility. 
On efficiency, our method leverages the KD framework, resulting in no increase in Params. or FLOPs. Consequently, we achieve both improved accuracy and fast inference speed.
\begin{table*}[t]
\centering
\renewcommand\arraystretch{1.2}
\setlength{\tabcolsep}{5pt}{
\caption{Comparisons on mIoU (\%) with state-of-the-art methods on the \emph{val} sets of Pascal VOC 2012~\cite{everingham2010pascal} and Cityscapes~\cite{cordts2016cityscapes}, and ADE20K~\cite{zhou2017scene}. ``$\ddag$'' denotes our re-implemented result based on the released codes due to inconsistencies in experimental settings. 
``KD Manner'':  knowledge distillation manner, which contains of layer-to-layer (L2L) and whole-to-whole (W2W) as illustrated in Figure~\ref{cdmanner}.
``Ground-truth Mask'': ground-truth mask used to obtain the pre-extraction image boundaries. ``${\emph{E}}$'': physical edge. ``${\emph{B}}$'': semantic boundary. ``${\emph{P}}$'': pixel-wise. ``${\emph{I}}$'': image-wise. ``${\emph{O}}$'': object-wise.}
\label{rtab3}
\begin{tabular}{r|c|c|c|c|r|r|r}
\hline \hline 
\multicolumn{5}{r}{$\mathbb{T}$: PSPNet-101~\cite{zhao2017pyramid}} & 77.82\% & 78.56\% & 42.19\% \\
\multicolumn{5}{r}{$\mathbb{S}$: PSPNet-18$_{(1.0)}$~\cite{zhao2017pyramid}} & 70.78\% & 69.10\% & 33.82\% \\
\hline 
Methods & KD Manner &Ground-truth mask? & Boundary type & Context type & Pascal VOC 2012  & Cityscapes & ADE20K\\
\hline 
+ KD~\cite{hinton2015distilling} & L2L &\xmark & \xmark & \xmark & 71.28$^\ddag$$_{\color{red}{+0.50}}$  & 71.20$_{\color{red}{+2.10}}$ & 34.33$^\ddag$$_{\color{red}{+0.51}}$ \\ 
+ SKD~\cite{liu2019structured} & L2L & \cmark & ${\emph{E}}$ & ${\emph{P}}$ &  73.05$_{\color{red}{+2.27}}$  & 71.45$_{\color{red}{+2.35}}$ & 34.65$_{\color{red}{+0.83}}$ \\ 
+ SCKD~\cite{zhu2021student} & L2L & \xmark & \xmark & \xmark & 72.33$_{\color{red}{+1.55}}$ & 72.10$_{\color{red}{+3.00}}$ & 34.76$_{\color{red}{+0.94}}$ \\ 
+ CIRKD~\cite{yang2022cross} & L2L & \xmark & \xmark & ${\emph{I}}$ &  73.57$^\ddag$$_{\color{red}{+2.79}}$ & 72.25$_{\color{red}{+3.15}}$ & 34.93$_{\color{red}{+1.11}}$ \\ 
+ IFD~\cite{chen2022improved}& L2L  & \xmark & \xmark & \xmark &  73.88$^\ddag$$_{\color{red}{+3.10}}$ & 72.63$_{\color{red}{+3.53}}$ & 35.15$^\ddag$$_{\color{red}{+1.33}}$ \\
+ FGKD~\cite{yang2022focal}& L2L  & \xmark & \xmark & ${\emph{O}}$ &  72.90$^\ddag$$_{\color{red}{+2.12}}$ & 72.55$_{\color{red}{+3.45}}$ & 35.24$_{\color{red}{+1.42}}$  \\ 
+ CWT~\cite{liu2023simple} & L2L & \xmark & \xmark & ${\emph{O}}$ &  73.06$^\ddag$$_{\color{red}{+2.28}}$  & 72.60$_{\color{red}{+3.50}}$ & 35.21$^\ddag$$_{\color{red}{+1.39}}$ \\ 
+ SlimSeg~\cite{xue2022slimseg} & L2L & \cmark & ${\emph{E}}$ & \xmark &  74.08$_{\color{red}{+3.30}}$ & 73.95$_{\color{red}{+4.85}}$ & 37.12$_{\color{red}{+3.30}}$ \\
+ BGLSSeg~\cite{zhou2024boundary}& L2L & \cmark & ${\emph{E}}$ & \xmark & 74.27$^\ddag$$_{\color{red}{+3.49}}$ & 74.10$^\ddag$$_{\color{red}{+5.00}}$ & 36.49$^\ddag$$_{\color{red}{+2.67}}$ \\
+ FAM~\cite{pham2024frequency}& L2L  & \xmark & \xmark & \xmark &  74.28$^\ddag$$_{\color{red}{+3.50}}$ & 74.25$_{\color{red}{+5.15}}$ & 36.82$^\ddag$$_{\color{red}{+3.00}}$ \\
+ CrossKD~\cite{wang2024crosskd} & L2L & \xmark & \xmark & \xmark & 74.28$^\ddag$$_{\color{red}{+3.50}}$ & 74.28$_{\color{red}{+5.18}}$ & 36.72$^\ddag$$_{\color{red}{+2.90}}$\\
+ BPKD~\cite{liu2024bpkd} & L2L & \cmark & ${\emph{E}}$ & \xmark &  74.30$^\ddag$$_{\color{red}{+3.52}}$ & 74.29$_{\color{red}{+5.19}}$ & 37.07$^\ddag$$_{\color{red}{+3.25}}$ \\
\cellcolor[gray]{.95}\textbf{+ BCKD$_{\textrm{ours}}$}& \cellcolor[gray]{.95}W2W  & \cellcolor[gray]{.95}\xmark & \cellcolor[gray]{.95}${\emph{B}}$ & \cellcolor[gray]{.95}${\emph{P}}$ & \cellcolor[gray]{.95}74.65$_{\color{red}{+3.87}}$ & \cellcolor[gray]{.95}74.92$_{\color{red}{+5.82}}$ & \cellcolor[gray]{.95}37.62$_{\color{red}{+3.80}}$ \\ 
\hline 
+ IFVD~\cite{wang2020intra}& L2L  & \xmark & \xmark & \xmark &  74.05$_{\color{red}{+3.27}}$  & 74.41$_{\color{red}{+5.31}}$ & 36.63$^\ddag$$_{\color{red}{+3.35}}$ \\ 
\cellcolor[gray]{.95}\textbf{+ IFVD + BCKD$_{\textrm{ours}}$} & \cellcolor[gray]{.95}W2W & \cellcolor[gray]{.95}\xmark & \cellcolor[gray]{.95}${\emph{B}}$ & \cellcolor[gray]{.95}${\emph{P}}$ &  \cellcolor[gray]{.95}74.82$_{\color{red}{+4.04}}$  & \cellcolor[gray]{.95}74.99$_{\color{red}{+5.89}}$ & \cellcolor[gray]{.95}37.73$_{\color{red}{+3.91}}$ \\ 
+ TAT~\cite{lin2022knowledge}& L2L  & \xmark & \xmark & ${\emph{O}}$ &  74.02$^\ddag$$_{\color{red}{+3.24}}$ & 74.48$_{\color{red}{+5.38}}$ & 37.12$_{\color{red}{+3.30}}$ \\ 
\cellcolor[gray]{.95}\textbf{+ TAT + BCKD$_{\textrm{ours}}$}& \cellcolor[gray]{.95}W2W  & \cellcolor[gray]{.95}\xmark & \cellcolor[gray]{.95}${\emph{B}}$ & \cellcolor[gray]{.95}${\emph{P}}$\&${\emph{O}}$ &  \cellcolor[gray]{.95}74.71$_{\color{red}{+3.93}}$  & \cellcolor[gray]{.95}74.97$_{\color{red}{+5.87}}$ & \cellcolor[gray]{.95}38.12$_{\color{red}{+4.30}}$ \\ 
+ SSTKD~\cite{ji2022structural} & L2L & \cmark & ${\emph{E}}$ & \xmark &  73.91$_{\color{red}{+3.13}}$  & 74.60$_{\color{red}{+5.50}}$ & 37.22$_{\color{red}{+3.40}}$ \\ 
\cellcolor[gray]{.95}\textbf{+ SSTKD + BCKD$_{\textrm{ours}}$} & \cellcolor[gray]{.95}W2W & \cellcolor[gray]{.95}\cmark & \cellcolor[gray]{.95}${\emph{B}}$\&${\emph{E}}$ & \cellcolor[gray]{.95}${\emph{P}}$ &  \cellcolor[gray]{.95}74.52$_{\color{red}{+3.74}}$ & \cellcolor[gray]{.95}74.90$_{\color{red}{+5.80}}$ & \cellcolor[gray]{.95}38.24$_{\color{red}{+4.52}}$  \\ 
\hline \hline 
\end{tabular}}
\end{table*}

\subsubsection{Visualized comparisons} 
The visualized comparisons on the SSeg task with the baseline teacher and student models, and large models without the KD strategy are given in Figure~\ref{fig3}.
As highlighted by the white bounding boxes, the obtained results on Pascal VOC 2012~\cite{everingham2010pascal}, Cityscapes~\cite{cordts2016cityscapes}, ADE20K~\cite{zhou2017scene}, and COCO-Stuff 10K~\cite{caesar2018coco} demonstrate that BCKD yields significant improvements on both the boundary region completeness and the target region connectivity, when compared with the small student model's results. For example, the ``\emph{cow}'' in Pascal VOC 2012, the ``\emph{guidepost}'' in Cityscapes, the ``\emph{desk}'' and the ``\emph{TV bench}'' in ADE20K, the ``\emph{bus}'', the ``\emph{tennis racket}'', and the ``\emph{guideboard}'' in COCO-Stuff 10K. The results obtained are basically the same as those of the large teacher model.
Besides, compared to large models with higher model complexity (\ie, DANet~\cite{fu2019dual} and HRNet~\cite{sun2019high}) on Pascal VOC 2012 and Cityscapes, although our method is not as competitive as theirs on quantitative results, our method achieves better predictions on object boundaries and small objects, which validate the effectiveness and emphasize the importance of {boundary distillation} and {context distillation}. 
With the help of our method, the student model is also able to predict better masks for certain fine-grained objects. For example, the ``\emph{cow's ear}'' and the ``\emph{person's leg}''.

\subsection{Comparisons With SOTA Methods on SSeg}
In this section, we explore the superiority accuracy and the effectiveness of the joint implementation of BCKD with the state-of-the-art (SOTA) KD methods on SSeg. To ensure a fair comparison, PSPNet-101~\cite{zhao2017pyramid} and DeepLabV3 Plus-101~\cite{chen2018encoder} are employed as the teacher models, while PSPNet-18$_{(1.0)}$ and DeepLabV3 Plus-18~\cite{chen2018encoder} serve as the student models. The specific settings for each student model are described in detail in the provided table. Some results are re-implemented by us on the released code due to inconsistencies in experimental settings and are marked with ``$\ddag$'' in the given tables. 
\begin{table}[t]
\centering
\renewcommand\arraystretch{1.2}
\setlength{\tabcolsep}{7pt}{
\caption{Comparisons on mIoU (\%), MFS and LHD with state-of-the-art methods on the \emph{test} set of COCO-Stuff 10K~\cite{caesar2018coco}.}
\label{rtab4}
\begin{tabular}{r|l|cc}
\hline \hline 
\multicolumn{1}{r}{$\mathbb{T}$: DeepLabV3 Plus-101} & 33.10\% & 1.00$\pm$0.00 & 0.00$\pm$0.00 \\
\multicolumn{1}{r}{$\mathbb{S}$: DeepLabV3 Plus-18} & 26.33\% & 0.45$\pm$0.08 & 4.21$\pm$0.35 \\
\hline 
Methods & mIoU (\%) & MFS ($\rho_l$) & LHD\\
\hline 
 + KD & 27.21$_{\color{red}{+0.88}}$ & 0.63$\pm$0.07 & 3.82$\pm$0.41 \\ 
 + SKD & 27.27$_{\color{red}{+0.94}}$ & 0.67$\pm$0.06 & 3.77$\pm$0.38 \\ 
 + SCKD & 27.38$_{\color{red}{+1.05}}$ & 0.72$\pm$0.08 & 3.63$\pm$0.42 \\ 
 + CIRKD & 27.68$_{\color{red}{+1.35}}$ & 0.81$\pm$0.05 & 3.45$\pm$0.33 \\ 
 + IFD & 27.91$_{\color{red}{+1.58}}$ & 0.88$\pm$0.04 & 3.28$\pm$0.29 \\
 + FGKD & 28.00$_{\color{red}{+1.67}}$ & 0.91$\pm$0.03 & 3.12$\pm$0.31 \\ 
 + CWT & 28.18$_{\color{red}{+1.85}}$ & 0.95$\pm$0.02 & 2.98$\pm$0.22 \\ 
 + IFVD & 28.35$_{\color{red}{+2.02}}$ & 1.02$\pm$0.03 & 2.77$\pm$0.25\\ 
 + C2VKD & 28.42$_{\color{red}{+2.09}}$ & 1.08$\pm$0.04 & 2.63$\pm$0.24\\
 + FAM & 28.48$_{\color{red}{+2.15}}$ & 1.12$\pm$0.03 & 2.55$\pm$0.21\\
 + CrossKD & 28.53$_{\color{red}{+2.20}}$ & 1.15$\pm$0.02 & 2.48$\pm$0.19\\
 + SSTKD & 28.70$_{\color{red}{+2.37}}$ & 1.23$\pm$0.03 & 2.31$\pm$0.17\\ 
\cellcolor[gray]{.95}\textbf{+ BCKD} & \cellcolor[gray]{.95}29.22$_{\color{red}{+2.89}}$ & \cellcolor[gray]{.95}1.41$\pm$0.02 & \cellcolor[gray]{.95}1.89$\pm$0.15\\ 
\hline 
 + TAT & 28.74$_{\color{red}{+2.41}}$ & 1.25$\pm$0.04 & 2.25$\pm$0.18 \\ 
\cellcolor[gray]{.95}\textbf{ + TAT + BCKD} & \cellcolor[gray]{.95}29.29$_{\color{red}{+3.11}}$ & \cellcolor[gray]{.95}1.44$\pm$0.01 & \cellcolor[gray]{.95}1.82$\pm$0.14\\ 
 + SlimSeg & 28.50$_{\color{red}{+2.17}}$ & 1.13$\pm$0.03 & 2.52$\pm$0.20 \\
\cellcolor[gray]{.95}\textbf{ + SlimSeg + BCKD} & \cellcolor[gray]{.95}29.45$_{\color{red}{+3.12}}$ & \cellcolor[gray]{.95}1.46$\pm$0.01 & \cellcolor[gray]{.95}1.79$\pm$0.13\\ 
 + BPKD & 28.66$_{\color{red}{+2.33}}$ & 1.21$\pm$0.02 & 2.36$\pm$0.16\\
\cellcolor[gray]{.95}\textbf{ + BPKD + BCKD} & \cellcolor[gray]{.95}29.60$_{\color{red}{+3.27}}$ & \cellcolor[gray]{.95}1.49$\pm$0.01 & \cellcolor[gray]{.95}1.72$\pm$0.12\\ 
\hline \hline 
\end{tabular}}
\vspace{-3mm}
\end{table}

\subsubsection{Superiority of BCKD} Compared to the SOTA KD methods on SSeg, on the top half of Table~\ref{rtab3} and Table~\ref{rtab4}, we can observe that our BCKD can surpass these methods. BCKD boosts the student model by $3.87\%$, $5.82\%$, and $3.80\%$ mIoU$\uparrow$ on the \emph{val} sets of Pascal VOC 2012~\cite{everingham2010pascal}, Cityscapes~\cite{cordts2016cityscapes}, and ADE20K~\cite{zhou2017scene}, respectively. Compared to the current SOTA KD methods on these datasets, BCKD outperforms IFVD~\cite{wang2020intra}, TAT~\cite{lin2022knowledge}, and SSTKD~\cite{ji2022structural} on Pascal VOC 2012 by 0.6\%, 0.63\%, and 0.74\% mIoU$\uparrow$, respectively. The visualized comparison results with the classic KD~\cite{hinton2015distilling} and SOTA IFVD methods are presented in the last row of Figure~\ref{fig4}. It can be observed that our method demonstrates significant advantages in capturing the connectivity of small objects as well as the integrity of the boundary masks.
Furthermore, BCKD achieves higher mIoU than SOTA methods on Cityscapes and ADE20K datasets as well. On the COCO-Stuff 10K~\cite{caesar2018coco} datasets in Table~\ref{rtab4}, our method surpasses the student model and the SOTA TAT model by 2.89\% and 0.48\% mIoU$\uparrow$, respectively. As demonstrated in~Table~\ref{rtab4}, the proposed BCKD framework also exhibits significant advantages in terms of MFS ($\rho_l$) and LHD. These results not only indicate an overall performance improvement but also validate the effectiveness of our novel boundary distillation and context distillation, which were designed to address these critical aspects of the learning process.
Since inference is only conducted on the student model, our method does not introduce any increase in model complexity. These results across different datasets can confirm that the task-specific knowledge is indeed more effective in practice compared to general knowledge.

\subsubsection{Effectiveness of the joint implementation} The experimental results on the joint implementation of BCKD and SOTA KD methods are presented on the lower half of Table~\ref{rtab3} and Table~\ref{rtab4}, respectively. 
It can be observed that, on top of BCKD, further adding IFVD~\cite{wang2020intra}, TAT~\cite{lin2022knowledge}, and SSTKD~\cite{ji2022structural} yields consistent performance gains, with mIoU$\uparrow$ improvements of $0.77\%$, $0.69\%$, and $0.61\%$ on the \emph{val} set of Pascal VOC 2012, respectively. This can be attributed to the fact that BCKD contains semantic boundary and context that is not present in these methods, further demonstrating the importance of semantic boundary and context information for the SSeg task. However, adding SSTKD~\cite{ji2022structural} on top of BCKD resulted in a performance decrease (\ie, 0.13\% mIoU$\downarrow$ on Pascal VOC 2012 and 0.02\% mIoU$\downarrow$ on Cityscapes) compared to the accuracy on BCKD. We guess that this may be because SSTKD uses superficial image texture information, which is non-semantic and contain some noise relative to the extracted semantic boundaries. On COCO-Stuff 10K, we can observe that our method further enhances the performance of all SOTA methods, including TAT~\cite{lin2022knowledge}, SlimSeg~\cite{xue2022slimseg}, and BPKD~\cite{liu2024bpkd}, and finally achieves 29.60\% mIoU on the \emph{test} set. 
\begin{figure*}[t]
\centering
\includegraphics[width=.99\textwidth]{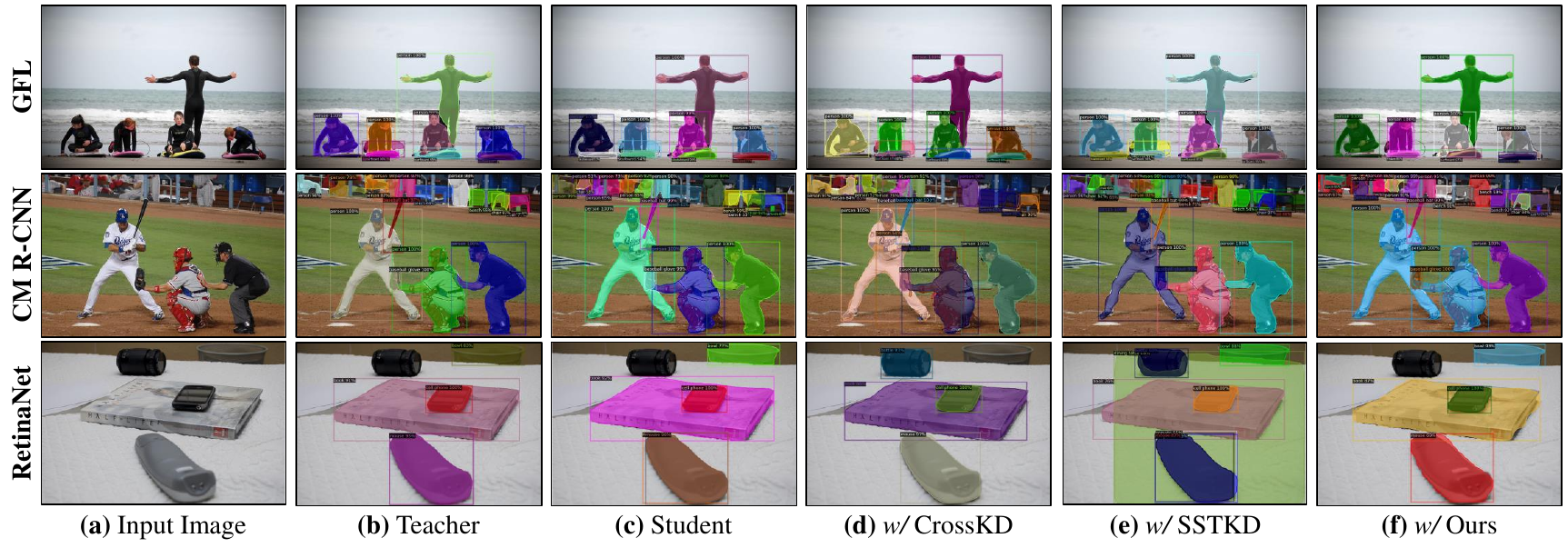}
\vspace{-3mm}
\caption{Visualization results on ISeg and ODet. ``\textsl{w/}'' means ``with'', indicating that the corresponding knowledge distillation method is deployed based on the student model. We chose the state-of-the-art methods CrossKD~\cite{wang2024crosskd} and SSTKD~\cite{ji2022structural} for comparison.}
\vspace{-4mm}
\label{figinsobd}
\end{figure*}

\begin{table}[t]
\centering
\renewcommand\arraystretch{1.2}
\setlength{\tabcolsep}{3pt}{
\caption{Result comparisons with the state-of-the-art methods on the \emph{val} set of MS-COCO 2017 \cite{lin2014microsoft} for ISeg and ODet. ``CM R-CNN'': Cascade Mask R-CNN. mAP$^\textrm{m}$ and mAP$^\textrm{b}$ denotes the average precision on instance segmentation mask and object detection bounding box, respectively.}
\label{rtable5}
\begin{tabular}{r|ccc|cc} 
\hline \hline 
Methods & AP$^\textrm{m}$ (\%) & AP$^\textrm{b}$ (\%) & FPS & MFS ($\rho_l$) & LHD \\ 
\hline 
$\mathbb{T}$: GFL-50~\cite{li2020generalized} & 36.8 & 40.2 & 19.4 & 1.00 & 0.00 \\
\hline 
$\mathbb{S}$: GFL-18~\cite{li2020generalized} & 33.1 & 35.8 & 23.7 & 0.82 & 3.20 \\ 
+ FGD~\cite{yang2022focal} & 34.0 & 36.6 & 23.7 & 0.85 & 2.90 \\ 
+ SKD~\cite{liu2019structured} & 34.3 & 36.9 & 23.7 & 0.86 & 2.70 \\ 
+ GID~\cite{dai2021general} & 34.6 & 37.8 & 23.7 & 0.89 & 2.50 \\ 
+ LD~\cite{zheng2022LD} & 34.8 & 38.0 & 23.7 & 0.90 & 2.30 \\ 
+ PKD~\cite{cao2022pkd} & 35.0 & 38.0 & 23.7 & 0.91 & 2.20 \\ 
+ CrossKD~\cite{wang2024crosskd} & 35.3 & 38.1 & 23.7 & 0.92 & 1.90 \\ 
+ SSTKD~\cite{ji2022structural} & 35.2 & 38.3 & 23.7 & 0.93 & 1.80 \\ 
\rowcolor[gray]{0.95} + BCKD$_{\textrm{ours}}$ & \textbf{35.8} & \textbf{38.8} & 23.7 & \textbf{0.95} & \textbf{1.60} \\ 
\hline 
$\mathbb{T}$: CM R-CNN-101~\cite{cai2018cascade} & 37.3 & 42.9 & 13.1 & 1.00 & 0.00 \\ 
\hline 
$\mathbb{S}$: CM R-CNN-50~\cite{cai2018cascade} & 36.5 & 41.9 & 16.1 & 0.88 & 2.10 \\ 
+ FGD~\cite{yang2022focal} & 35.3 & 42.1 & 16.1 & 0.87 & 2.20 \\ 
+ SKD~\cite{liu2019structured} & 36.5 & 42.2 & 16.1 & 0.89 & 2.00 \\ 
+ GID~\cite{dai2021general} & 36.7 & 42.0 & 16.1 & 0.90 & 1.90 \\ 
+ LD~\cite{zheng2022LD} & 36.8 & 42.1 & 16.1 & 0.91 & 1.80 \\ 
+ PKD~\cite{cao2022pkd} & 36.8 & 42.0 & 16.1 & 0.92 & 1.70 \\ 
+ CrossKD~\cite{wang2024crosskd} & 36.9 & 42.2 & 16.1 & 0.93 & 1.60 \\ 
+ SSTKD~\cite{ji2022structural} & 37.0 & 42.2 & 16.1 & 0.94 & 1.50 \\ 
\rowcolor[gray]{0.95} + BCKD$_{\textrm{ours}}$ & \textbf{37.0} & \textbf{42.5} & 16.1 & \textbf{0.96} & \textbf{1.40} \\ 
\hline 
$\mathbb{T}$: RetinaNet-101~\cite{ross2017focal} & 33.5 & 38.9 & 13.5 & 1.00 & 0.00 \\ 
\hline 
$\mathbb{S}$: RetinaNet-50~\cite{ross2017focal} & 31.7 & 37.4 & 17.7 & 0.85 & 2.40 \\ 
+ FGD~\cite{yang2022focal} & 32.1 & 37.7 & 17.7 & 0.86 & 2.20 \\ 
+ SKD~\cite{liu2019structured} & 32.5 & 37.5 & 17.7 & 0.87 & 2.12 \\ 
+ GID~\cite{dai2021general} & 32.8 & 37.6 & 17.7 & 0.88 & 2.03 \\ 
+ LD~\cite{zheng2022LD} & 33.1 & 37.8 & 17.7 & 0.89 & 1.90 \\ 
+ PKD~\cite{cao2022pkd} & 33.0 & 37.8 & 17.7 & 0.90 & 1.88 \\ 
+ CrossKD~\cite{wang2024crosskd} & 33.2 & 38.0 & 17.7 & 0.91 & 1.75 \\ 
+ SSTKD~\cite{ji2022structural} & 33.1 & 38.1 & 17.7 & 0.92 & 1.66 \\ 
\rowcolor[gray]{0.95} + BCKD$_{\textrm{ours}}$ & \textbf{33.3} & \textbf{38.5} & 17.7 & \textbf{0.94} & \textbf{1.54} \\ 
\hline \hline 
\end{tabular}}
\vspace{-5mm}
\end{table}

\subsection{Comparisons With SOTA Methods on ISeg and ODet}
The quantitative result comparisons on ISeg and ODet are presented in Table~\ref{rtable5}. The obtained results indicate that our method can consistently outperform existing methods across various baseline models, demonstrating its strong generalization and versatility. Specifically, we achieve AP scores of $35.8$\%/$38.8$\%, $37.0$\%/$42.5$\%, and $33.3$\%/$38.5$\% for instance segmentation masks (\ie, AP$^\textrm{m}$) and object detection bounding boxes (\ie, AP$^\textrm{b}$) on the GFL-18~\cite{li2020generalized}, Cascade Mask R-CNN-50~\cite{cai2018cascade}, and RetinaNet-50~\cite{ross2017focal}, respectively. In comparison with the SOTA CrossKD~\cite{wang2024crosskd} and SSTKD~\cite{ji2022structural}, our method demonstrates an average performance improvement of approximately $0.5$\%. This enhancement serves to validate the effectiveness of our proposed method. 
The results also demonstrate significant advantages in both MFS and LHD, indicating its superior capability in preserving the structure of learned feature manifolds while maintaining high precision in boundary-sensitive tasks. These results substantiate the effectiveness of our boundary and context distillation in maintaining geometric consistency and minimizing alignment errors.  

The visual comparison results with baseline methods and SOTA methods are shown in Figure~\ref{figinsobd}. It is observed that, relative to the baseline student models, the application of various KD strategies enhances the prediction results for specific classes (\eg, the \emph{person}'', the \emph{book}'', and the \emph{surfboard}''), thereby affirming the effectiveness of KD in dense image prediction tasks. Moreover, when compared to the SOTA methods CrossKD and SSTKD, our method demonstrates improved connectivity in object regions and boundary integrity (\eg, the \emph{person}'' and the \emph{baseball bat}''), highlighting the effectiveness of our proposed context distillation and boundary distillation strategies tailored for the targeted tasks. Additionally, our method addresses the issue of overlapping predicted bounding boxes (\eg, the \emph{mouse}'' and the ``\emph{chair}''), a benefit attributed to the enriched contextual information incorporated into the student model via context distillation. 

Furthermore, we also observed a significant phenomenon wherein our method effectively reduces the occurrence of hallucinations in the student model's predictions. Specifically, as depicted in the last column of Figure~\ref{figinsobd}, both the teacher and student models fail to identify the ``\emph{camera}'', while the CrossKD and SSTKD methods mistakenly classify the ``\emph{camera}'' as the ``\emph{bottle}''. In contrast, our approach accurately recognizes the ``\emph{camera}'' as a background object, aligning with the dataset's definitions. We hypothesize that this discrepancy may stem from the confusion of target knowledge caused by task-irrelevant KD during the training process. Our proposed task-specific BCKD is inherently designed to alleviate such confusion from the outset.
\section{Conclusion and Future Work}
In this work, we propose a customized boundary and context knowledge distillation (BCKD) method tailored for efficient dense image prediction tasks on AI accelerator, including semantic segmentation, instance segmentation, and object detection. Our approach significantly narrows the performance gap between compact, efficient models and their larger, more accurate counterparts while maintaining computational efficiency. Specifically, BCKD enhances boundary-region completeness and ensures object-region connectivity, leading to consistent accuracy improvements across diverse challenging benchmarks and architectures. Theoretical analysis further corroborates the effectiveness  of our method. 

As a generalizable method, in the future, we plan to extend BCKD to additional dense visual tasks (\eg, pose estimation and image synthesis) and investigate its adaptation to emerging architectures (\eg, Vision Transformer and Vision Mamba) to better support model compression for AI accelerator deployment. Moreover, we will explore synergies between BCKD and large foundation models (\eg, Segment Anything Model and vision-language models) to further enhance the robustness of lightweight dense predictors under adverse conditions.

\bibliographystyle{IEEEtran}
\bibliography{main}
\end{document}